\documentclass{article}

\usepackage[main, final]{neurips_2026}

\usepackage[utf8]{inputenc} 
\usepackage[T1]{fontenc}    
\usepackage{hyperref}       
\usepackage{url}            
\usepackage{booktabs}       
\usepackage{amsfonts}       
\usepackage{nicefrac}       
\usepackage{longtable}
\usepackage{microtype}      
\usepackage{xcolor,graphicx}         
\usepackage{algorithm}
\usepackage{algorithmic}
\usepackage{amsthm,amsmath,amssymb,mathtools}

\theoremstyle{plain}
\newtheorem{theorem}{Theorem}[section]
\newtheorem{proposition}[theorem]{Proposition}

\theoremstyle{definition}
\newtheorem{definition}[theorem]{Definition}
\newtheorem{assumption}[theorem]{Assumption}
\theoremstyle{remark}
\newtheorem{remark}[theorem]{Remark}

\newcommand{\revise}[1]{{\textcolor{black}{#1}}}
\newcommand{\tanspace}[1]{T_{#1} \Manifold}
\newcommand{\gt}{g_\theta}
\newcommand{\yw}[1]{\textcolor{black}{#1}}

\newcommand{\reb}[1]{\textcolor{black}{#1}}

\title{Riemannian Lyapunov Framework: Optimization as Closed Loop Control on Riemannian Manifolds}

\author{%
  Yixuan Wang\thanks{Corresponding author: \texttt{wang.yixuan@ufl.edu}.}\\
  University of Florida
  \And
  Omkar Sudhir Patil\\
  Louisiana State University
  \And
  Warren E. Dixon\\
  Virginia Tech
}

\newcommand{\Manifold}{\mathcal{M}}

\newcommand{\R}{\mathbb{R}}

\newcommand{\grad}{\operatorname{grad}}
\newcommand{\norm}[2]{\| #1 \|_{#2}}

\newcommand{\Retr}{R_\theta}
\newcommand{\Trans}[2]{\mathcal{T}_{#1 \rightarrow #2}}

\begin{document}

\maketitle

\begin{abstract}
%
\reb{We introduce the Riemannian Lyapunov Framework (RLF), a unified framework that represents first order optimization as a closed loop controlled dynamical system on a Riemannian parameter manifold.}
Within this framework, optimizers are specified by a quadruple $(M, g, \Phi, \eta)$ consisting of a Riemannian metric, an internal state transition map, a direction field generator, and a lifting parameter.
Classical and modern methods including SGD, Adam, AdamW, Lion and Shampoo are recovered as specific instantiations of this quadruple.
The central structural object of the framework is a time indexed family of target submanifolds in the extended state space of parameters, velocities, and internal memory, which we call the Normally Attracting Invariant Manifold (NAIM) family.
These submanifolds organize training dynamics into two timescales: a fast contraction of the velocity state onto the current target submanifold, followed by a slow descent along it.
Using Riemannian backstepping, we construct a strict Lyapunov function on the extended state space and prove Uniform Ultimate Boundedness (UUB) of the optimization trajectory under a descent alignment assumption and bounded stochastic disturbances.
Under the additional Polyak-Lojasiewicz (PL) condition, the Lyapunov function converges linearly to a bounded neighborhood of the optimum, defined as the drift of the target direction field between consecutive iterates after parallel transport.
To validate the framework's generativity, we derive three concrete instantiations of the quadruple, each exercising a distinct design axis, and show on geometric diagnostics, image classification benchmarks, and language modeling that their empirical behavior matches the theoretical predictions.
The framework casts optimizer design as controller synthesis, complementing heuristic perspectives by providing a constructive template with provable guarantees.
\end{abstract}

\section{Introduction}
The evolution of deep learning optimization has produced a vast ecosystem of algorithms designed to navigate complex loss landscapes \cite{keskar2016large, li2018visualizing, bottou2018optimization}.
This lineage includes Stochastic Gradient Descent (SGD) \cite{Robbins1951ASA, chaudhari2019entropy}, momentum methods \cite{POLYAK19641, pmlr-v28-sutskever13}, Nesterov’s Accelerated Gradient (NAG) \cite{botev2017nesterov}, and a diverse array of adaptive methods such as AdaGrad \cite{JMLR:v12:duchi11a}, RMSProp \cite{1370017282431050757}, Adam \cite{kingma2014adam}, AdamW \cite{loshchilov2017decoupled}, and Adafactor \cite{shazeer2018adafactor}. Recently, specialized techniques like Lion \cite{chen2023symbolic}, Shampoo \cite{gupta2018shampoo}, and Sophia \cite{liu2023sophia} have pushed the boundaries of efficiency in large scale training by utilizing sign based updates \cite{bernstein2018signsgd, karimireddy2019error} or second order information \cite{yao2021adahessian}.
While these methods are empirically excellent and form the backbone of modern machine learning, they are largely treated as a collection of distinct heuristics.
Theoretical analyses of these optimizers remain fragmented, often focusing on narrow properties or specific algorithmic artifacts \cite{reddi2019convergence, chen2018closing} while failing to provide a global, principled explanation for their shared success.

Previous attempts to unify accelerated methods often \yw{leverage} continuous time limits that approximate discrete iterations by second order damped dynamics, providing interpretable surrogate models for rate and stability \cite{su2015differential, wibisono2016variational}.
Other works have explored unified frameworks through the lens of proximal operators \cite{Boyd2011Distributed} or mirror descent \cite{BECK2003167}, yet these perspectives typically analyze algorithms through fixed time rescalings or simplified state representations\yw{,} and therefore\yw{,} do not directly synthesize feedback laws that explicitly regulate the evolution of auxiliary optimizer states.
As a result, they do not provide a constructive mechanism level explanation for the observed separation between fast residual alignment and slower descent, especially in discrete time with time varying preconditioning and stochastic gradients.
This lack of a cohesive geometric and control theoretic foundation limits our ability to systematically design new optimizers that are both stable and effective across different architectures.

In this paper, we answer the key open question: \textit{Is there a unified geometric principle that governs the stability and efficacy of these diverse optimizers?}
%
%
\reb{We introduce the Riemannian Lyapunov Framework (RLF), which provides a unified control perspective for understanding and designing machine learning optimizers.}
Our motivation is grounded in the observation of two timescale dynamics where the system state must track a specific relationship between its velocity and the gradient field.
We formalize this using the concept of a Normally Attracting Invariant Manifold (NAIM), which serves as the geometric skeleton relating the update speed to the target direction field.
Drawing inspiration from the idea of backstepping in nonlinear control theory \cite{slotine1991applied, krstic1995nonlinear, dixon2003control}, we construct a strict Lyapunov function that encodes both the objective value and the distance to the NAIM.
This approach allows us to derive a controller that actively forces the system to track the manifold, transforming optimizer design from heuristic experimentation into a principled synthesis of Lyapunov based feedback laws.
Our contributions lie on:

\textbf{(i) Unified Riemannian Geometry:}
We demonstrate that diverse optimization components are equivalent to fundamental geometric objects where preconditioning corresponds to selecting a Riemannian metric and momentum represents an extended velocity state in the tangent bundle.

\textbf{(ii) The NAIM Mechanism:} We introduce the NAIM as the primary structure organizing training dynamics, proving that optimization consists of a fast process of backstepping the speed state to a target graph followed by a slow drift along that manifold.

\textbf{(iii) Strict Lyapunov Design:} We provide a systematic methodology for designing the RLO family of optimizers by synthesizing control laws that satisfy strict Lyapunov requirements using a Riemannian backstepping approach, ensuring robust convergence even under the disturbances typical of stochastic gradients.

This framework provides a blueprint for optimization algorithm design from the perspective of controller design and paves the way for principled innovations.

\section{Related Work}
\reb{Lyapunov methods provide a common basis for stability analysis and feedback design \cite{krstic1995nonlinear}.
 The analysis constructs a Lyapunov function $V$ and studies its derivative along system trajectories to establish stability or convergence properties.
Such functions also support analyses of accelerated and adaptive optimization dynamics \cite{Su2016diff,Andre2020general}.
This perspective motivates constructing optimizer updates that jointly regulate objective descent and velocity tracking.}

\reb{NAIM describe attraction toward an invariant manifold and evolution along it \cite{FENICHEL197953, eldering2018global}.
Their transverse contraction dominates the tangent dynamics.
Related analyses of fast and slow dynamics arise in singular perturbation theory and stochastic algorithms with two time scales \cite{kokotovic1999singular, BORKAR2025104759}.
For optimizer design, this geometric perspective motivates separating the generation of a target direction from the dynamics that track it.}

\section{Geometric Setup and Problem Formulation}
\label{sec:setup}
We study the minimization of a differentiable objective function $f\colon \Manifold\rightarrow \R$ over a smooth $n$ dimensional parameter manifold $\Manifold$ endowed with a Riemannian metric $g$.
The optimal value of $f$ is denoted by $f^*$, which may be unknown.
Discrete time is indexed by $k\in\mathbb{N}$.
\subsection{The Riemannian Configuration}
For each $\theta\in\Manifold$, let $\tanspace{\theta}$ denote the tangent space at $\theta$.
The Riemannian metric assigns an inner product $g(\cdot, \cdot)$ on $\tanspace{\theta}$, and the induced norm is 
\begin{align*}
    \norm{\xi}{\gt} \triangleq \sqrt{g_\theta(\xi,\xi)}, \quad \xi \in \tanspace{\theta}.
\end{align*}
The Riemannian gradient of $f$ at $\theta$, denoted $\grad f(\theta) \in \tanspace{\theta}$, is defined by the identity
\begin{align*}
    d f(\theta)[\xi] = \gt(\grad f(\theta), \xi)
\end{align*}
for all $\xi\in \tanspace{\theta}$, where $d f(\theta)[\xi]$ is the directional derivative of $f$ at $\theta$ along $\xi$.

To perform iterative optimization, we require a mechanism to map tangent vectors back onto the manifold and to compare vectors across distinct tangent spaces.
We employ a retraction $R$ and a vector transport $\mathcal{T}$ as our primary computational operators.\footnote{When $\Manifold = \R^n$ with the Euclidean metric, one has $\tanspace{\theta} = \R^n$, $\Retr(\xi) = \theta +\xi$ and $\Trans{\theta}{\varphi}(\xi) = \xi$.
All definitions above reduce to their standard vector space counterparts.
%
}
A retraction is a smooth mapping $R\colon \tanspace{} \rightarrow \Manifold$ such that for any $\theta$, the restriction $\Retr\colon \tanspace{\theta} \rightarrow \Manifold$ satisfies $R_\theta(0) = \theta$ and its local differential at the origin is the identity map.
%
%
Furthermore, for any two points $\theta$ and $\varphi = \Retr(\xi)$, the vector transport $\Trans{\theta}{\varphi}\colon \tanspace{\theta} \rightarrow \tanspace{\varphi}$ provides a linear mapping that moves a tangent vector from the source space to the destination space along the retraction curve.
%

%
%
%
\begin{figure}
    \centering
    \includegraphics[width=0.7\linewidth]{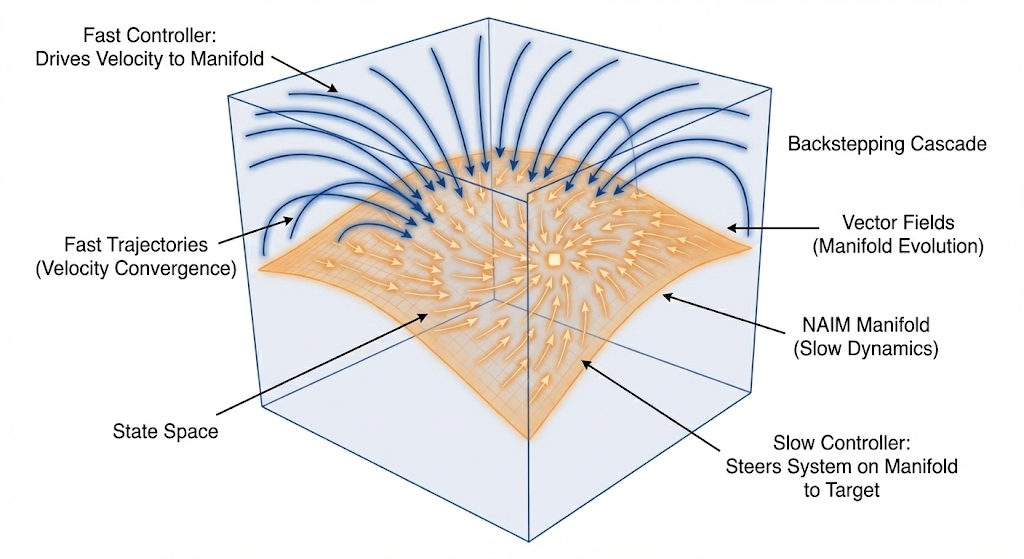}
    \caption{Geometric intuition of the NAIM Lyapunov framework.
    The orange surface represents the NAIM embedded in the extended state space $\theta, v$.
    Blue streamlines illustrate the fast dynamics: the velocity state $v$ rapidly contracts onto $\Lambda$ at a rate governed by the lifting parameter $\eta$, regardless of initial conditions.
    Orange arrows show the slow dynamics: once the trajectory reaches $\Lambda$ , the system evolves along the direction field $\Phi$ toward the target optimum at the center.
    %
    }
    \label{fig:intuition}
\end{figure}
\subsection{Optimization as an Extended Dynamical System}
We conceptualize the optimizer not merely as a rule for updating parameters, but as a controlled dynamical system evolving in an extended state space. We define the system state at time $k$ as the tuple $x_k = (\theta_k, v_k, y_k)$.
Here, $\theta_k \in \Manifold$ represents the current model parameters.
The variable $v_k \in \tanspace{\theta_k}$ represents the physical velocity or momentum of the system, residing in the tangent bundle.
The variable $y_k \in \mathcal{Y}$ represents the internal memory state of the optimizer, such as the accumulated first or second moments of the gradients, where $\mathcal{Y}$ is a vector space appropriate for the \revise{selected} statistics.

The evolution of this system is governed by \revise{an internal state transition map $\Psi$ and a direction field generator $\Phi$.}
The transition map $\Psi: \mathcal{Y} \times T\Manifold \to \mathcal{Y}$ updates the memory state based on the current stochastic gradient.
The direction field generator $\Phi: \mathcal{Y} \times T\Manifold \to T\Manifold$ constructs a target update direction $d_k = \Phi(y_k, \grad f(\theta_k))$ in the tangent space $\tanspace{\theta_k}$.
This map $\Phi$ encapsulates the structural logic of the optimizer, including operations such as normalization, clipping, or sign-based transformations.
Consequently, a specific optimization algorithm is fully characterized by the tuple $(\Manifold, g, \Phi, \Psi)$ alongside a sequence of step sizes $h_k$ and lifting parameters $\eta_k$.
We rigorously demonstrate in Appendix \ref{app:precondition} that adaptive preconditioning methods in Euclidean space are mathematically equivalent to the selection of a specific time-varying Riemannian metric $g$, thereby bringing adaptive optimizers within this geometric framework.
\section{The NAIM-Lyapunov Framework}
\label{sec:naim}
Building upon the geometric formulation of the optimizer as an extended dynamical system, we adopt a constructive control perspective.
We begin with a generic open-loop mechanical system on the manifold and ask: \textit{What control law is required to force the system to track the NAIM while minimizing the objective function?}
The answer is obtained through Riemannian backstepping, a Lyapunov based synthesis within the proposed framework under a strict Lyapunov requirement.
Fig \ref{fig:intuition} illustrates the geometric intuition underlying the framework.
\subsection{Open-Loop Dynamics and the Geometric Objective}
We model the optimization process as a mechanical system on the tangent bundle $T\Manifold$.
The open-loop dynamics in continuous time are governed by the covariant equations:
\begin{align*}
    \dot{\theta} = -v, \quad \nabla_{\dot{\theta}} v = u,
\end{align*}
where $\theta \in \Manifold$ is the position, $v \in \tanspace{\theta}$ is the velocity, $\nabla$ is the Levi-Civita connection associated with metric $g$, and $u \in \tanspace{\theta}$ is the control input (acceleration) we must design.
Our geometric objective is twofold.
First, we require the system to minimize the loss $f(\theta)$.
Second, we require the physical velocity $v$ to adhere to the direction field generator introduced in Section \ref{sec:setup}.
This defines the NAIM graph
\begin{align*}
    \Lambda(y) = \{(\theta, v) \in \tanspace{\theta} \colon v = \Phi(y, \grad f)\}.
\end{align*}
and the normal residual measuring the deviation from this graph is defined as: $z\triangleq v - \Phi(y, \grad f) \in \tanspace{\theta}$.
The control problem is to synthesize an input $u$ that renders $\Lambda$ attractive and invariant while ensuring $f(\theta)$ decreases.

For clarity, when indexed by a discrete time step $k$ we write $\Lambda_k = \Lambda(y_k)$ and $z_k = v_k - \Phi(y_k, \grad f(\theta_k))$.
The full extended state $(\theta, v, y)$ lies in $T\Manifold \times \mathcal{Y}$, and the NAIM is a $y$ indexed family of submanifolds in this full space rather than a single static set.
\subsection{Lyapunov-Based Control Synthesis}
\label{sec:sub lyapunov control}
To synthesize the control law $u$, we postulate a strict Lyapunov function candidate $V$ that encodes a weighted sum of the potential energy (loss) and the kinetic energy of the residual,
\begin{align*}
    V(\theta, v, y) = f(\theta) - f^\star+\|z\|_{g_\theta}^2.
\end{align*}
Taking the time derivative of $V$ along the trajectories of the open-loop system \revise{yields}
\begin{align*}
    \dot{V} = \langle \grad f, \dot{\theta} \rangle_{g_\theta} + 2 \langle z, \nabla_{\dot{\theta}} z \rangle_{g_\theta} = -\langle \grad f, v \rangle_{g_\theta} + 2\langle z, \nabla_{\dot{\theta}} z \rangle_{g_\theta}.
\end{align*}
Substituting $v = \Phi + z$ and $\nabla_{\dot{\theta}} z = \nabla_{\dot{\theta}} v - \nabla_{\dot{\theta}} \Phi = u - \dot{\Phi}$, we obtain
\begin{align*}
    \dot{V} = \underbrace{-\langle \grad f, \Phi \rangle_{g_\theta}}_{\text{natural descent}} \underbrace{- \langle \grad f, z \rangle_{g_\theta}}_{\text{cross coupling}} + 2 \langle z, u - \dot{\Phi} \rangle_{g_\theta}.
\end{align*}
The first term represents the natural descent along the direction field, which is non-positive whenever $\Phi$ aligns with $\grad f$ in the sense of Assumption \ref{assum:descent align}.
The second term is a sign indefinite cross coupling between the descent and residual subsystems, arising because the residual carries a position dependent component.
The third term is the control term that will be used to enforce contraction on the residual.
%
%
%
%
%
Following the principle of backstepping, we treat the direction field $\Phi$ as a virtual reference prescribed on the NAIM and design controller $u$ to achieve two goals simultaneously: (a) cancel the sign indefinite cross coupling, and (b) inject strict dissipation into the residual dynamics at a prescribed rate.
Accordingly, we choose the control law:
\begin{align*}
    u =\nabla_{\dot{\theta}}\Phi - \frac{1}{\tau} z.
\end{align*}
Here, $\tau > 0$ is the relaxation time constant.
\revise{The first term $\nabla_{\dot{\theta}}$ provides the necessary feedforward acceleration to track the evolving target manifold and the second term 
$- z/\tau$ injects strict damping to contract the residual fiber at a rate determined by the time constant $\tau$.}

%

Then taking the time-derivative of the Lyapunov function yields
\begin{align*}
    \dot{V} &= -\langle \grad f, \Phi \rangle_{g_\theta} - \langle \grad f, z \rangle_{g_\theta} + 2 \langle z, -\frac{z}{\tau} \rangle_{g_\theta}=-\langle \grad f, \Phi \rangle_{g_\theta} - \langle \grad f, z \rangle_{g_\theta} - \frac{2}{\tau} \norm{z}{g_\theta}^2.
\end{align*}
By Assumption \ref{assum:descent align} and Young's inequality, for any $\alpha>0$, we have
\begin{align*}
    \dot{V} &=-\langle \grad f, \Phi \rangle_{g_\theta} - \langle \grad f, z \rangle_{g_\theta} - \frac{2}{\tau} \norm{z}{g_\theta}^2 \leq -(\mu_\Phi - \frac{\alpha}{2})\norm{\grad f}{g_\theta}^2 - ( \frac{2}{\tau} - \frac{1}{2\alpha}) \norm{z}{g_\theta}^2.
\end{align*}
Note that $\dot{V}$ is strictly less than $0$, by picking $\alpha = \mu_\Phi$ and $\tau < 4\mu_\Phi$.

We discretize the closed-loop dynamic $\nabla_\theta v = \dot{\Phi} - \frac{1}{\tau}z$ using a first-order Euler discretization with time step $h_k > 0$.
The continuous relaxation $1/\tau$ maps to the discrete lifting parameter $\eta_k \in (0,1]$ through $\eta_k = h/\tau$, and the constraint $\tau < 4\mu_\Phi$ translates to $h < 4\mu_\Phi \eta_k$.
The position update $\dot{\theta} = v$ discretizes via the retraction:
\begin{align*}
    \theta_{k+1} = R_{\theta_k}(-h_k \tilde{v}_{k+1}).
\end{align*}
\revise{where $\tilde{v}_{k+1}$ is the lifted speed update computed from $v_k$ and $v_{k+1}$ is subsequently updated based on $\tilde{v}_{k+1}$.}
\begin{align*}
    \tilde{v}_{k+1} = (1-\eta_k)v_k + \eta_k \Phi(y_k, g_k).
\end{align*}
Thus, the optimization algorithm designed based on this framework is the numerical \revise{approximation} of the feedback controller required to stabilize the NAIM.
\subsection{Lyapunov Stability and Uniform Ultimate Boundedness}
In practical settings, the feedforward term $\nabla_{\dot{\theta}}\Phi$ is often approximated \revise{using stochastic gradients, which introduces disturbances in the resulting closed-loop dynamical system.}
To certify stability under these non-vanishing disturbances, we 
%
analyze the discrete Lyapunov difference $\Delta V_k := V_{k+1} - V_k$.
\yw{The goal is to develop a} deterministic guarantee that trajectories are confined within a compact set, referred to as a \revise{``}thick tube," even under worst-case bounded disturbances.
\begin{theorem}
\label{thm:uub}
     Consider the RLO system under Definition \ref{assum:L smooth}, Assumption \ref{assum:bounded}, \ref{assum:descent align} and \ref{assum:lipschitz} regarding smoothness, isometric vector transportation, descent alignment and Lipschitz condition.
     Consider the discrete RLO dynamics with step size $h_k$ and lifting parameter $\eta_k \in (0, 1]$.
     Then the following results are obtained.
     
        \textbf{Thick Tube.}
          There exists a weighting constant $\alpha > 0$ such that for sufficiently small $h_k$, the discrete Lyapunov function candidate
          \begin{align*}
              V_k(\theta_k, z_k) \triangleq f(\theta_k) - f^\star + \|z_k\|_{g_{\theta_k}}^2
          \end{align*}
          satisfies the difference inequality:
          \begin{align*}
              V_{k+1} - V_k &\leq C_1 h_k \|\grad f(\theta_k)\|_{g_{\theta_k}}^2- C_2\|z_k\|_{g_{\theta_k}}^2 + C_3 h_k^2,
          \end{align*}
          where $C_1, C_2, C_3$ are \yw{constants}.
          This implies that the residual $z_k$ is uniformly bounded by a region proportional to the learning rate $h_k$, effectively confining the trajectory to a ``thick tube" around $\Lambda$.
          
          \textbf{Uniformly Ultimate Boundedness.}
          In addition, if $f$ satisfies the Polyak-Lojasiewicz (PL) condition (\ref{def:plcondition}) with constant $\mu_{\text{PL}} > 0$, then the Lyapunov function $V_k$ converges linearly to a noise floor determined by:
          \begin{align*}
              V_{k+1} \leq (1 - \rho) V_k +C_3 h_k^2,
          \end{align*}
          where $\rho \propto h_k \in (0, 1)$ is the linear contraction rate.
          Consequently, the optimization error is ultimately bounded by:
          \begin{align*}
              \limsup_{k \to \infty} (f(\theta_k) - f^\star) \leq \limsup_{k \to \infty} V_k \le \frac{C_3 h_k}{\mu_\Phi \mu_{PL}} = O(h_k).
          \end{align*}
\end{theorem}
A complete proof, together with all required assumptions, is provided in Appendix \ref{app:proof}.
This result rigorously validates the \revise{``}Thick Tube" mechanism: the optimizer establishes a dynamic equilibrium where the Lyapunov contraction balances the disturbances.
The tube radius $\sqrt{\delta_{\max}}$ is explicitly minimized by the smoothness of $\Phi$, theoretically justifying why smooth direction fields yield lower steady-state loss than discontinuous ones.
\section{Unification of Modern Optimizers}
\label{sec:uni}
The constructive control analysis in Section \ref{sec:naim} yields a specific discrete-time feedback controller required to stabilize the NAIM.
We formally encapsulate this control logic into the RLO, presented in Algorithm \ref{alg:rlo}.
This algorithm serves not merely as a new method, but as a universal computational template determined by four geometric components: the metric $g$, the internal state transition $\Psi$, the direction field generator $\Phi$, and the lifting parameter $\eta$.
\begin{algorithm}[t]
   \caption{Riemannian Lyapunov Optimizer (RLO)}
   \label{alg:rlo}
\begin{algorithmic}[1]
   \STATE {\bfseries Input:} Initial $\theta_0$, $v_0=0$, state $y_0$. Metric $g$, Map $\Phi$, Transition $\Psi$.
   \STATE {\bfseries Hyperparameters:} Step sizes $\{h_k\}$, lifting rates $\{\eta_k\}$.
   \FOR{$k = 0, 1, \dots, K-1$}
      \STATE \textbf{// Geometric Phase (Target Construction)}
      \STATE Obtain stochastic gradient $\hat{g}_k \in T_{\theta_k}\mathcal{M}$.
      \STATE $y_{k+1} \gets \Psi(y_k, \hat{g}_k)$ \hfill \COMMENT{Update internal memory}
      \STATE $d_k \gets \Phi(y_{k+1}, \hat{g}_k)$ \hfill \COMMENT{Define target vector on $\Lambda_k$}
      
      \STATE \textbf{// Dynamic Phase (Manifold Tracking)}
      \STATE $\tilde{v}_{k+1} \gets (1 - \eta_k) v_k + \eta_k d_k$ \hfill \COMMENT{Fiber contraction}
      \STATE $\theta_{k+1} \gets R_{\theta_k}(-h_k \tilde{v}_{k+1})$ \hfill \COMMENT{Parameter update}
      \STATE $v_{k+1} \gets \mathcal{T}_{\theta_k \to \theta_{k+1}}(\tilde{v}_{k+1})$ \hfill \COMMENT{Vector transport}
   \ENDFOR
\end{algorithmic}
\end{algorithm}

Algorithm~\ref{alg:rlo} proceeds in two distinct phases that mirror the two time scales of the dynamics. 

The \textbf{Geometric Phase} (Lines 3--5) constructs the target manifold $\Lambda_k$.
%
%
%
And the \textbf{Dynamic Phase} (Lines 6--8) enforces the \yw{convergence update}.
%
%
This structure strictly enforces the separation of time scales: $\Phi$ defines the manifold, while $\eta$ defines the attraction rate.
By instantiating the components $(\Manifold, g, \Phi, \eta)$, Algorithm~\ref{alg:rlo} reproduces standard optimizers.
We provide a comprehensive mapping in Table~\ref{tab:unification_full} (Appendix~\ref{app:unification}), but highlight the structural equivalences here\yw{.}

\textbf{Adaptive Preconditioning as Metric Selection.}
Methods like Adam and RMSProp are rigorously recovered by selecting a time-varying Riemannian metric $g_k = \text{diag}(\sqrt{s_k})^{-1}$, where $s_k$ is the second moment estimate.
Under this metric, the standard Adam update is not a heuristic scaling, but the canonical fiber contraction evolving on a geometry that adapts to the local curvature.

\textbf{Symbolic Optimizers as \revise{Nonlinear Gradient} Fields.}
Algorithms like Lion utilize the sign operator, breaking the link between update magnitude and gradient norm.
In our framework, this corresponds to a nonlinear direction field $\Phi(y, g) = \text{sign}(y)$.
While this implies $\Lambda$ is now a nonlinear graph, the Lyapunov stability analysis remains valid provided the forcing term $\delta_k$ remains bounded.
%
%
%
%
%

\section{Experiments}
\label{sec:geometric_validation}

We evaluate RLF through three complementary studies.
We first examine tracking and Lyapunov diagnostics, then isolate the effect of Riemannian metric selection within a fixed RLO template.
The remaining three components of the framework are examined through three synthesized algorithms and are shown in Appendix \ref{app:unification}.
We finally compare representative RLO instances with baseline optimizers across image classification and language modeling tasks.
Experimental settings are provided in Appendices~\ref{app:exp_setup}.

\begin{table}[h]
\centering
\caption{Mechanism diagnostics on CIFAR 10 with ResNet 18.
Validation accuracy is reported as mean $\pm$ standard deviation over three seeds.
%
}
\label{tab:mechanism_diagnostics}
\small
\setlength{\tabcolsep}{2pt}
\begin{tabular}{@{}lccccccc@{}}
\toprule
$\Phi$ & $h$ & Best val. (\%) & Final val. (\%) & Tail $\|z_k\|_2$ & Tail $\|\delta_k\|_2$ & Mean $\Delta\widehat V_k$ & Tail $\widehat V_k$ (max) \\
\midrule
$\tanh$ & $3\times10^{-4}$ & $48.37\pm1.04$ & $48.35\pm1.03$ & $0.84$ & $0.96$ & $-3.8\times10^{-5}$ & $1.44\;(1.47)$ \\
$\tanh$ & $10^{-3}$ & $65.01\pm0.51$ & $64.97\pm0.47$ & $0.84$ & $0.95$ & $-3.1\times10^{-5}$ & $1.22\;(1.26)$ \\
$\tanh$ & $3\times10^{-3}$ & $71.09\pm1.25$ & $59.24\pm0.69$ & $0.88$ & $1.00$ & $+2.5\times10^{-5}$ & $1.84\;(1.86)$ \\
$\tanh$ & $10^{-2}$ & $72.52\pm3.25$ & $67.71\pm3.42$ & $0.86$ & $0.97$ & $+1.7\times10^{-5}$ & $1.45\;(1.48)$ \\
\midrule
$\operatorname{sign}$ & $3\times10^{-4}$ & $91.69\pm0.01$ & $91.35\pm0.32$ & $3339$ & $3801$ & $-4.2\times10^{-5}$ & $1.10\;(1.45)$ \\
$\operatorname{sign}$ & $10^{-3}$ & $91.37\pm0.23$ & $90.70\pm0.11$ & $2227$ & $2534$ & $-6.0\times10^{-5}$ & $0.73\;(4.46)$ \\
$\operatorname{sign}$ & $3\times10^{-3}$ & $89.27\pm0.41$ & $81.67\pm10.73$ & $1151$ & $1308$ & $+6.4\times10^{-4}$ & $4.46\;(31.26)$ \\
$\operatorname{sign}$ & $10^{-2}$ & $80.51\pm1.02$ & $29.89\pm12.28$ & $747$ & $843$ & $+2.1\times10^{-4}$ & $10.40\;(52.74)$ \\
\bottomrule
\end{tabular}
\end{table}

\begin{table}[t]
\centering
\caption{Metric ablation on CIFAR-10 with ResNet-18, all framework components held fixed except the Riemannian metric. Each variant is evaluated at its individually selected peak learning rate. Best Acc and Final Acc are reported as mean $\pm$ standard deviation across three seeds.}
\label{tab:metric_ablation}
\small
\begin{tabular}{lcccc}
\toprule
Metric    & Form                & LR    & Best Acc (\%)      & Final Acc (\%)\\
\midrule
RMSProp   & Diagonal            & $3\times 10^{-3}$ & $94.38\pm 0.07$ & $94.38\pm 0.07$ \\
AdaGrad   & Diagonal cumulative & $1\times 10^{-2}$ & $93.96\pm 0.13$ & $93.96\pm 0.13$ \\
Adam      & Diagonal debiased   & $1\times 10^{-3}$ & $93.86\pm 0.15$ & $93.84\pm 0.13$ \\
Layer-wise & Block diagonal      & $3\times 10^{-3}$ & $91.57\pm 0.10$ & $91.53\pm 0.10$ \\
Euclidean & Identity            & $3\times 10^{-2}$ & $81.72\pm 0.17$ & $81.72\pm 0.17$ \\
\bottomrule
\end{tabular}
\end{table}

We emphasize that the RLO family is introduced solely to evaluate the behavior of RLF and its optimizer instances and to provide a control theoretic blueprint for future optimizer design, and not as a comparison against current state of the art optimizers.

\subsection{Mechanism Diagnostics}
\label{subsec:mechanism_diagnostics}

We examine how changes in the target direction sustain tracking error, as characterized by the residual dynamics in Appendix~\ref{app:detail proof}.
We also assess how the direction field affects optimization behavior across learning rates.
We use the lifted RLO update in Algorithm~\ref{alg:smooth_lifted_rlo} with global normalization disabled and compare its tanh and sign variants.
We fix $\eta=0.7$ and weight decay at $0.1$, disable global normalization, and hold the remaining optimizer settings fixed.
We evaluate constant learning rates in $\{3\times10^{-4},10^{-3},3\times10^{-3},10^{-2}\}$.
Validation uses a held out subset of $5{,}000$ training images.

We record the tracking residual $z_k$ and the direction increment $\delta_k$.
And we record the Lyapunov diagnostic $\widehat V_k\triangleq f_{\mathrm{probe}}(\theta_k)+\|z_k\|_2^2/D$ to examine objective behavior together with tracking.
Here, $f_{\mathrm{probe}}$ is the loss on a fixed probe batch with frozen batch normalization statistics, and $D$ is the parameter count.
Its increment is $\Delta\widehat V_k\triangleq\widehat V_{k+1}-\widehat V_k$.
To examine behavior beyond the initial transient, we report tail statistics over the final half of training.
In Table~\ref{tab:mechanism_diagnostics}, ``tail'' denotes the mean over this window, and parentheses report the maximum over the same window.
The residual recursion $z_{k+1}=(1-\eta)z_k-\delta_k$ identifies target variation as the forcing that sustains tracking error.

The tanh field exhibits smaller tail means of both the forcing norm and the tracking residual across the tested learning rates.
For both fields, the ratio of the reported tail mean tracking norm to the tail mean forcing norm lies between $0.875$ and $0.887$.
This consistent relationship supports the interpretation of a tracking tube sustained by target variation.
The Lyapunov function distinguishes the subsequent behavior of the two fields.
For tanh, its tail mean remains between $1.22$ and $1.84$ across the sweep.
For sign, the two larger learning rates produce substantially larger tail levels and maxima, accompanied by lower final validation accuracy.
These observations reveal different empirical stability regimes for the two direction fields.
Individual positive Lyapunov increments are compatible with ultimate boundedness, so the tail levels and maxima complement the mean increment when assessing the observed dynamics.

\subsection{Metric Ablation}
\label{subsec:g abl}

Within the RLO template we fix the direction field $\Phi(c) = \tanh(\gamma c)$, the lifting parameter $\eta = 1$, the EMA decays $\beta_1 = 0.9$ and $\beta_2 = 0.99$, the saturation scale $\gamma = 5$, and the belief injection coefficient $\lambda_b = 0.2$, and we vary only the metric across five canonical choices.
The Euclidean metric corresponds to $A = I$.
The Adam metric uses the bias corrected diagonal $A_{k,i,i} = 1/(\sqrt{\hat v_{k,i}} + \varepsilon)$ with second moment update $v_k = \beta_3 v_{k-1} + (1-\beta_3)\,\hat g_k \odot \hat g_k$, debiased mean $\hat v_k = v_k/(1 - \beta_3^{k+1})$, and $\beta_3 = 0.999$.
The RMSProp metric uses the same diagonal form but without bias correction and with $\beta_3 = 0.99$.
The AdaGrad metric replaces the exponential moving average by the cumulative sum $G_k = \sum_{t \le k} \hat g_t \odot \hat g_t$ and sets $A_{k,i,i} = 1/(\sqrt{G_{k,i}} + \varepsilon)$.
The Layer-wise metric is block diagonal with one block per parameter tensor; each block is a scalar multiple of the identity, scaled by the inverse of that tensor's gradient norm.
The learning rate is selected after a grid search, and the final evaluation is performed on three independent seeds with the selected learning rate for 50 epochs.
The training protocol uses CIFAR-10 with the ResNet-18 architecture adapted to $32{\times}32$ inputs (the first convolution replaced by a $3{\times}3$ stride 1 layer with the initial max pooling removed, totaling 11.17M parameters), batch size 128, cosine learning rate schedule with a 5 epoch linear warmup, weight decay 0.1, label smoothing 0.1, and global gradient norm clipping at 1.0.

\begin{table}[h]
\centering
\caption{Baseline comparison on CIFAR-10 with ResNet-18.}
\label{tab:cifar10_baselines}
\small
\setlength{\tabcolsep}{12pt}
\begin{tabular}{lccc}
\toprule
Optimizer      & Best (\%) & Avg (\%) & Std \\
\midrule
AdamW          & 93.87 & 93.77 & 0.13 \\
Lion           & 94.03 & 93.88 & 0.13 \\
Adan           & 93.86 & 93.78 & 0.09 \\
Shampoo        & \textbf{94.92} & \textbf{94.83} & 0.09 \\
RLO-$\Lambda$     & 93.90 & 93.89 & \textbf{0.02} \\
RLO-Lifted     & 94.01 & 93.91 & 0.14 \\
\bottomrule
\end{tabular}
\end{table}

Table~\ref{tab:metric_ablation} reports the results.
The accuracy structure is organized by the topological form of the metric.
The three diagonal variants (Adam, RMSProp, AdaGrad) cluster between $93.86$ and $94.38$ percent.
The Layer-wise block diagonal variant trails at $91.57$ percent.
The Euclidean baseline reaches only $81.72$ percent, twelve percentage points below the diagonal cluster.
This stratification quantifies the contribution of the metric axis in isolation: changing only the metric accounts for a $12.66$ percentage point gap between the Euclidean variant and the best diagonal variant on this benchmark.
Within the diagonal cluster, the choice between bias corrected and uncorrected exponential moving averages is a second order effect with differences within or close to the across seed standard deviation.
The Layer-wise variant illustrates that the granularity of the metric is itself a design axis: collapsing each parameter tensor to a single scalar discards the coordinate level curvature anisotropy that the diagonal metrics retain.
 
This experiment evaluates the metric axis of RLF.
The remaining three design axes of the framework, namely the direction field generator $\Phi$, the lifting parameter $\eta$, and the global normalization mechanism, are each examined through dedicated controlled ablations in Appendix~\ref{app:exp_setup}.

\subsection{Performance Across Architectures and Modalities}
\label{sec:multidomain}
 
\textbf{CIFAR-10 Classification.}
All first order methods cluster within $0.17$ percentage points, indicating ResNet-18 on CIFAR-10 is near saturation for first order optimization; Shampoo's full second order information separates it at $94.92\%$.
RLO-$\Lambda$ exhibits the lowest across seed standard deviation among all six methods at $0.02$, more than $4.5\times$ tighter than the next lowest first order method, which is a direct empirical confirmation of Theorem~\ref{thm:uub}: the bounded Lipschitz $\tanh(\gamma c)$ field gives a smaller forcing $\delta_k$ than the discontinuous sign field, hence a tighter thick tube and lower across seed variance at convergence.
 
\textbf{ImageNet classification.}
Following the protocol of~\cite{chen2023symbolic}, we compare RLO variants against AdamW and Lion on ImageNet~\cite{russakovsky2015imagenet} using ResNet-50, ViT-S/16~\cite{touvron2021training}, and ViT-B/16~\cite{dosovitskiy2020image}.
All models train from scratch for 90 epochs at global batch size 1024 with cosine decay and 5 epoch linear warmup; learning rates and weight decay are tuned per optimizer per architecture.
On ResNet-50 all methods fall within $0.62$ percentage points (Table~\ref{tab:imagenet}).

On the Vision Transformers RLO-$\Lambda$ leads, achieving $76.18\%$ on ViT-S/16 ($1.0$ point above AdamW) and $76.47\%$ on ViT-B/16 ($5.05$ points above AdamW at $71.42\%$); RLO-$\Lambda$ is the only method to achieve the highest accuracy on all three architectures. Validation accuracy curves and the architecture sensitive analysis of RLO-Lifted are in Appendix~\ref{app:vitb16_dynamics}.
 
\begin{table}[h]
\centering
\caption{ImageNet Top-1 accuracy (\%). All models trained from scratch for 90 epochs.}
\label{tab:imagenet}
\small
\setlength{\tabcolsep}{10pt}
\begin{tabular}{lccc}
\toprule
Optimizer    & ResNet-50 & ViT-S/16 & ViT-B/16 \\
\midrule
AdamW        & 73.87 & 75.18 & 71.42 \\
Lion         & 73.36 & 75.14 & 76.27 \\
RLO          & 73.45 & 75.38 & 76.00 \\
RLO-$\Lambda$   & \textbf{73.98} & \textbf{76.18} & \textbf{76.47} \\
RLO-Lifted   & \textbf{73.98} & 71.43 & 76.33 \\
\bottomrule
\end{tabular}
\end{table}
 
\textbf{GPT-2 Small on OpenWebText.}
We train GPT-2 Small (124M) on OpenWebText for $10\,000$ gradient steps with $2\,000$ step linear warmup and cosine decay; AdamW uses peak learning rate $6{\times}10^{-4}$ and the four sign or momentum based optimizers use $6{\times}10^{-5}$~\cite{chen2023symbolic}.
RLO-$\Lambda$ achieves the lowest validation perplexity at $21.56$, a $7.42\%$ relative improvement over AdamW at $23.29$ (Table~\ref{tab:gpt2}).
The generalization gap $\mathrm{Gap} = L_{\mathrm{val}} - L_{\mathrm{train}}$ is small and roughly constant across methods (range $0.0035$ to $0.0049$), with RLO-$\Lambda$ exhibiting the largest gap, so the perplexity advantage stems from a lower training loss basin ($3.0658$ versus AdamW at $3.1432$), an optimization advantage rather than a regularization advantage.
Additional AdamW tuning results over peak learning rates and warmup durations are provided in Appendix~\ref{app:gpt2_adamw_search}.
 
\begin{table}[h]
\centering
\caption{GPT-2 Small on OpenWebText, $10\,000$ gradient steps. PPL denotes perplexity.}
\label{tab:gpt2}
\small
\setlength{\tabcolsep}{6pt}
\begin{tabular}{lccccc}
\toprule
Optimizer    & Val Loss & Best PPL & Final PPL & Train Loss & Generalization\ Gap \\
\midrule
AdamW        & 3.1476          & 23.29           & 23.28           & 3.1432            & 0.0044    \\
Lion         & 3.2012          & 24.53           & 24.56           & 3.1975            & 0.0037    \\
RLO          & 3.2019          & 24.57           & 24.58           & 3.1983            & \textbf{0.0036}    \\
RLO-$\Lambda$   & \textbf{3.0708} & \textbf{21.56}  & \textbf{21.56}  & \textbf{3.0658}   & 0.0049    \\
RLO-Lifted   & 3.1614          & 23.61           & 23.60           & 3.1579            & \textbf{0.0035}    \\
\bottomrule
\end{tabular}
\end{table}

\section{Limitation and Future Directions}
\reb{In this paper, RLF provides a foundation for optimizer design based on control theory and offers a unified control perspective on machine learning optimizers.
The present analysis focuses on establishing this connection and characterizing the stability and tracking properties of the resulting optimization dynamics.
The additional benefits that other control tools could offer for optimization remain to be investigated.
For example, an optimal control formulation could guide the selection of optimizer dynamics through a specified cost functional.
Such a cost could balance objective reduction, tracking accuracy, and update effort.
Integrating these additional control tools and studying their benefits for optimizer design remain directions for future work.}

\section{Conclusion}
\label{sec:conclusion}
This paper develops a unified geometric framework for understanding and designing optimization algorithms in machine learning.
By interpreting optimization as an extended-state control problem on a Riemannian manifold, we establish a principled foundation that reveals structural connections among seemingly disparate methods and enables systematic construction of new algorithms with provable guarantees.

This NAIM-Lyapunov framework characterizes optimizers through three geometric objects: a Riemannian metric $g$, a direction field generator $\Phi$, and a contraction rate $\eta$.
%
%
The framework thus provides a generative grammar for optimization algorithms rather than merely a post-hoc taxonomy.
The Lyapunov-based construction methodology transforms optimizer design from heuristic experimentation into strict Lyapunov controller design.
Beginning with open-loop gradient dynamics, we construct a Lyapunov function certifying convergence, derive a control law guaranteeing $\dot{V} < 0$, and obtain closed-loop dynamics whose discretization yields the update rule.
This synthesis ensures that resulting algorithms inherit stability properties from the continuous-time analysis.

%
Systematic ablation validates that the NAIM geometric structure provides optimization benefits independent of auxiliary mechanisms, while large-scale experiments demonstrate that RLO-$\Lambda$ achieves the best average performance over ImageNet classification benchmarks.

\newpage
\begin{ack}
\reb{The authors thank the reviewers for their careful reading of the manuscript and constructive suggestions, which substantially improved the clarity and rigor of the presentation.
The authors also extend special thanks to Dr. Dan P. Guralnik for his valuable insights into the mathematical foundations of NAIM.}

\reb{This research was supported in part by AFOSR grant FA9550-19-1-0169 and AFRL grant FA8651-21-F-1027. Any opinions, findings and conclusions or recommendations expressed in this material are those of the author(s) and do not necessarily reflect the views of the sponsoring agency.}
\end{ack}

\bibliography{references}
\bibliographystyle{unsrt}

\appendix
\onecolumn
\allowdisplaybreaks
\section{Notations}
\label{app:notation}

\begin{center}
\renewcommand{\arraystretch}{1.3}
\begin{longtable}{p{0.15\textwidth} p{0.8\textwidth}}
\caption{Summary of Notation} \label{tab:notation} \\
\toprule
\textbf{Symbol} & \textbf{Description} \\
\midrule
\multicolumn{2}{l}{\textit{Geometry and Manifold Structure}} \\
$\Manifold$ & The $d$ dimensional smooth parameter manifold. \\
$\tanspace{\theta}$ & The tangent space at $\theta \in \Manifold$. \\
$\gt(\cdot, \cdot)$ & Riemannian metric, an inner product on $T_\theta \Manifold$. \\
$\norm{\xi}{g_\theta}$ & Riemannian norm induced by $g$, that is $\sqrt{g_\theta(\xi, \xi)}$. \\
$\grad f(\theta)$ & Riemannian gradient of $f$ at $\theta$. \\
$R_\theta(\xi)$ & Retraction map $R_\theta: T_\theta \Manifold \to \Manifold$. \\
$\mathcal{T}_{\theta \to \phi}(\xi)$ & Vector transport map $\mathcal T_{\theta\to\phi}:T_\theta\mathcal M\to T_\phi\mathcal M$. \\
$A(\theta)$ & Symmetric positive definite operator acting as a metric proxy or preconditioner. \\
$\| \xi \|_{A, g}$ & Weighted norm $\sqrt{g_\theta(\xi, A(\theta)\xi)}$. \\
\midrule
\multicolumn{2}{l}{\textit{Extended Dynamics and NAIM}} \\
$k$ & Discrete time index. \\
$\theta_k$ & Parameter at index $k$. \\
$v_k$ & Velocity state in $T_{\theta_k} \Manifold$. \\
$y_k$ & Internal state, such as exponential moving averages and second moment estimates. \\
$d_k$ & Target direction field at step $k$, $d_k \in T_{\theta_k} \Manifold$. \\
$\Lambda_k$ & Target graph set, an example: $\Lambda_k \triangleq \{(\theta, v) : v = d_k\}$. \\
$z_k$ & Normal residual, $z_k \triangleq v_k - d_k$. \\
$\Delta d_{k+1}$ & Drift of the target field, $d_{k+1} - \mathcal{T}_{\theta_k\to\theta_{k+1}}(d_k)$. \\
\midrule
\multicolumn{2}{l}{\textit{Algorithm Parameters}} \\
$h_k$ & Step size at index $k$. \\
$\eta_k$ & Lift parameter, fiber contraction rate, $\eta_k \in (0, 1]$. \\
$\lambda_b$ & Belief injection weight. \\
$\gamma$ & Scaling factor inside the $\tanh$ nonlinearity for RLO-Lifted. \\
$\beta_1, \beta_2$ & Exponential moving average decay rates. \\
$a_k$ & Magnitude calibration scalar for the target field. \\
\midrule
\multicolumn{2}{l}{\textit{Diagnostics and Analysis}} \\
$V(\theta, v)$ & Strict Lyapunov function candidate on the extended state space. \\
$r_k$ & Relative residual thickness. \\
$q_k^\perp$ & Relative orthogonal forcing. \\
$\text{cos}_k$ & Alignment cosine similarity. \\
$\sigma^2$ & Conditional second moment bound of the stochastic gradient noise. \\
\bottomrule
\end{longtable}
\end{center}

\section{Preconditioning as Metric Selection}
\label{app:precondition}
We provide \yw{a} justification for interpreting adaptive preconditioning matrices (common in Adam, RMSProp, and AdaGrad) as Riemannian metrics.
This equivalence allows us to analyze these algorithms using geometric tools rather than treating preconditioning as an ad-hoc modification of the gradient.
\begin{proposition}
    Let $\Manifold = \R^n$ equipped with the standard Euclidean metric $\langle \cdot, \cdot \rangle_2$.
    Let $A(\theta)$ be a symmetric positive definite matrix field on $\Manifold$.
    Consider a modified gradient step of the form $\theta_{k+1} = \theta_k - h A(\theta_k) \nabla f(\theta_k)$, where $\nabla f$ is the standard Euclidean gradient.
    This update is equivalent to a Riemannian Gradient Descent step with respect to the metric $g^A_\theta(\xi, \zeta) = \langle \xi, A(\theta)^{-1} \zeta \rangle_2$.
\end{proposition}
\begin{proof}
    First, recall the definition of the Riemannian gradient $\grad f(\theta)$.
    It is the unique vector in $\tanspace{\theta}$ such that for all tangent vectors $\xi$
    \begin{align*}
        \langle \grad_{g^A} f(\theta), \xi \rangle_{g^A_\theta} = \revise{d} f(\theta)[\xi],
    \end{align*}
    where $\revise{d}f(\theta)[\xi] = \langle \nabla f(\theta), \xi \rangle_2$ is the directional derivative.
    Substituting the definition of the metric $g^A$
    \begin{align*}
        \langle \grad_{g^A} f(\theta), A(\theta)^{-1} \xi \rangle_2 = \langle \nabla f(\theta), \xi \rangle_2.
    \end{align*}
    Since $A(\theta)$ is symmetric, we can move the inverse term to the other side of the inner product
    \begin{align*}
        \langle A(\theta)^{-1} \grad_{g^A} f(\theta), \xi \rangle_2 = \langle \nabla f(\theta), \xi \rangle_2.
    \end{align*}
    Since this holds for all $\xi$, we must have
    \begin{align*}
        A(\theta)^{-1} \grad_{g^A} f(\theta) = \nabla f(\theta).
    \end{align*}
    Multiplying both sides by $A(\theta)$ yields
    \begin{align*}
        \grad_{g^A} f(\theta) = A(\theta) \nabla f(\theta).
    \end{align*}
    Therefore, the preconditioned update direction $A(\theta) \nabla f(\theta)$ is exactly the steepest descent direction under the metric induced by the inverse preconditioner $A(\theta)^{-1}$.
    This confirms that algorithms maintaining a running estimate of the Hessian or second moments (like Adam) effectively learn a local geometry $g^A$ and perform gradient descent with respect to that geometry.
\end{proof}
\section{Assumptions and Proof of Theorem \ref{thm:uub}}
\label{app:proof}
We provide a detailed derivation of the stability and convergence properties stated in Theorem \ref{thm:uub}.
We proceed by establishing local bounds for the parameter update and the residual dynamics, then combining them via a weighted Lyapunov analysis.
\subsection{Definitions and Assumptions}
We use the abbreviation $\| \cdot \|_k \triangleq \| \cdot \|_{g_k}$ and $\langle \cdot, \cdot \rangle_k \triangleq \langle \cdot, \cdot \rangle_{g_k}$.
\yw{Consider the following} geometric regularity conditions required for the analysis.

From \citep{AbsilMahonySepulchre+2008}[Definition 7.4.1] Chapter 7,  we have:
\begin{definition}[Riemannian $L$-Smoothness]
\label{assum:L smooth}
    The objective function $f: \Manifold \rightarrow \R$ is differentiable and has an $L_g$-Lipschitz continuous gradient with respect to the retraction $R$.
    Specifically, there exists a constant $L_g > 0$ such that for any $\theta \in \Manifold$ and update vector $\xi \in \tanspace{\theta}$:
    \begin{align*}
        f(R_\theta(\xi)) \leq f(\theta) + \langle \grad f(\theta), \xi \rangle_{g_\theta} + \frac{L_g}{2} \|\xi\|_{g_\theta}^2.
    \end{align*}
\end{definition}
\begin{assumption}[Isometric Vector Transport]
\label{assum:bounded}
    $\| \Trans{\theta}{R_\theta(\xi)}(u) \|_{k}^2  = \| u\|_{g_\theta}.$
    This is trivial in Euclidean space.
\end{assumption}
\begin{assumption}[Descent Alignment \& Boundedness]
\label{assum:descent align}
    The direction field generator $\Phi(y, g)$ is designed such that the target vector $d_k$ aligns with the negative gradient direction.
    We assume there exists a constant $\mu_\Phi > 0$ such that:
    \begin{align*}
        \langle \grad  f(\theta_k), d_k \rangle_{k} \ge \mu_\Phi \|\grad  f(\theta_k)\|_{k}^2.
    \end{align*}
    And the target vector is bounded by design:
    \begin{align*}
        \| d_k\|_g \leq D_{max}
    \end{align*}
    where $D_{max}$ is a constant.
\end{assumption}
\begin{remark}
    Assumption \ref{assum:descent align} formalizes the requirement that $\Lambda$ is a descent graph.
    For SGD ($d_k =\grad f$), $\mu_\Phi=1$.
\end{remark}
\begin{assumption}[Lipschitz regularity]
\label{assum:lipschitz}
$\Phi$ is $L_\Phi$ Lipschitz in the gradient slot, $L_\Phi^y$ Lipschitz in the internal state slot; $\Psi$ has stride $\|\Psi(y, u) - y\|_\mathcal{Y} \leq L_\Psi \|u\|_{g_\theta} h$; and $\Phi$ near commutes with $\mathcal{T}$ in the sense $\| \Phi(y,\mathcal{T}\grad f) - \mathcal{T} \Phi(y, \grad f)\|_{R_\theta(\xi)} \leq C_\Phi \|{\xi}_{g_\theta}\|$, with $C_\Phi = 0$ in Euclidean space.
Composite constant: $L_\delta \triangleq L_\Phi L_g + C_\Phi + L_\Phi^y L_\Psi / \mu_\Phi$.
\end{assumption}
\begin{definition}[Polyak-Lojasiewicz Condition.]
\label{def:plcondition}
    For the convergence analysis, we say $f$ satisfies the PL condition with constant $\mu_{\text{PL}} > 0$ if for all $\theta \in \Manifold$:
    \begin{align*}
        \frac{1}{2}\|\grad f(\theta)\|_k^2 \geq \mu_{\text{PL}} (f(\theta) - f^\star).
    \end{align*}
\end{definition}
\subsection{Detail Proof of Theorem \ref{thm:uub}}
\label{app:detail proof}
Consider the parameter update $\theta_{k+1} = R_{\theta_k}(-h_k \tilde{v}_{k+1})$.
Applying Definition \ref{assum:L smooth} with update vector $\xi = -h_k \tilde{v}_{k+1}$ yields
\begin{align*}
    f(\theta_{k+1}) \leq f(\theta_k) - h_k \langle \grad f(\theta_k), \tilde{v}_{k+1} \rangle_k + \frac{L_g h_k^2}{2} \|\tilde{v}_{k+1}\|_k^2.
\end{align*}
Recall the lifted velocity update: $\tilde{v}_{k+1} = (1-\eta_k)v_k + \eta_k d_k$.
By definition of the residual $z_k = v_k - d_k$, substitute $v_k = d_k + z_k$ into the velocity update yields:
\begin{align}
\label{eqn tilde v}
    \tilde{v}_{k+1} = (1-\eta_k)(d_k + z_k) + \eta_k d_k = d_k + (1-\eta_k)z_k.
\end{align}
Then the inner product term gives:
\begin{align*}
    \langle \grad f, \tilde{v}_{k+1} \rangle_k = \langle \grad f, d_k \rangle_k + (1-\eta_k)\langle \grad  f, z_k \rangle_k.
\end{align*}
Using the triangle inequality and Assumption \ref{assum:descent align}:
\begin{align*}
    \| \tilde{v}_{k+1}\|_k^2 = \| d_k +(1-\eta_k) z_k\|_k^2 \leq 2 \| d_k\|_k^2 +2(1-\eta_k)^2 \|z_k\|_k^2 \leq 2 D_{max}^2 + 2\| z_k\|_k^2\yw{,}
\end{align*}
Using Assumption \ref{assum:descent align}, we have $\langle \grad  f, d_k \rangle_{g_k} \geq \mu_\Phi \|\grad  f\|_{g_k}^2$.
Thus:
\begin{align}
\label{eqn:f diff}
    f(\theta_{k+1}) - f(\theta_k) &\leq -h_k \mu_\Phi \|\grad  f\|_{k}^2 - h_k(1-\eta_k)\langle \grad  f, z_k \rangle_{k}+ L_g h_k^2 D^2_{max}+L_g h_k^2\| z_k\|_k^2.
\end{align}
Bounding the cross coupling by Young's inequality with parameter $\mu_\Phi / (1-\eta_k)$ for $\eta_k <1$,
\begin{align*}
    - h_k(1-\eta_k)\langle \grad  f, z_k \rangle_{k} \leq \frac{h_k \mu_\Phi}{2} \|\grad f \|_k^2 +\frac{h_k(1-\eta_k)^2}{2\mu_\Phi} \| z_k\|_k^2,
\end{align*}
we obtain,
\begin{align*}
    f(\theta_{k+1}) - f(\theta_k) &\leq -\frac{h_k \mu_\Phi}{2} \|\grad  f\|_{k}^2 + L_g h_k^2 D^2_{max}+(L_g h_k^2 + \frac{h_k(1-\eta_k)}{2\mu_\Phi})\| z_k\|_k^2.
\end{align*}

The residual at the next step is $z_{k+1} = v_{k+1} - d_{k+1}$.
From the algorithm, $v_{k+1} = \mathcal{T}_{\theta_k \to \theta_{k+1}}(\tilde{v}_{k+1})$.
Thus:
\begin{align*}
    z_{k+1} = \mathcal{T}(d_k + (1-\eta_k)z_k) - d_{k+1}.
\end{align*}
Using the linearity of $\mathcal{T}$, we obtain
\begin{align*}
    z_{k+1} = (1-\eta_k)\mathcal{T}(z_k) - \underbrace{(d_{k+1} - \mathcal{T}(d_k))}_{\delta_k}.
\end{align*}
Here, $\delta_k$ is the Riemannian forcing term.
Taking the squared norm on both sides yields
\begin{align*}
    \|z_{k+1}\|_{k+1}^2 = (1-\eta_k)^2 \|\mathcal{T}(z_k)\|_{k+1}^2 - 2(1-\eta_k)\langle \mathcal{T}(z_k), \delta_k \rangle_{k+1} + \|\delta_k\|_{k+1}^2.
\end{align*}
Using Young's Inequality on the cross term, for every $\eta_k \in (0,1]$, we have
\begin{align*}
    -2\langle (1-\eta_k)\mathcal{T}(z_k), \delta_k \rangle_{k+1} &\leq (1-\eta_k)^2 (1+ \frac{\eta_k}{2}) \|\mathcal{T}(z_k)\|_{k+1}^2 + (1 +\frac{2}{\eta_k}) \|\delta_k\|_{k+1}^2\\
    &\leq (1-\frac{\eta_k}{2}) \|\mathcal{T}(z_k)\|_{k+1}^2 + \frac{3}{\eta_k} \|\delta_k\|_{k+1}^2\yw{.}
\end{align*}
Using Assumption \ref{assum:bounded}, namely $\|\mathcal{T}(z_k)\|_{k+1} = \|z_k \|_k$, and combining all terms yields the residual contraction inequality
\begin{align}
\label{eqn:force bound}
    \|z_{k+1}\|^2 \le (1 - \frac{\eta_k}{2}) \|z_k\|_k^2 + \frac{3}{\eta_k} \|\delta_k\|_{k+1}^2.
\end{align}

Using $d_k = \Phi(y_k,\grad f(\theta_k))$ the forcing $\delta_k$ can be decomposed as:
\begin{align*}
    \delta_k &= \Phi(y_{k+1},\grad f(\theta_{k+1})) - \mathcal{T}\Phi(y_{k},\grad f(\theta_{k}))\nonumber\\
    &=\Phi(y_{k+1},\grad f(\theta_{k+1})) - \Phi(y_{k+1},\mathcal{T}\grad f(\theta_k))\\
    &+ \Phi(y_{k+1},\mathcal{T}\grad f(\theta_k)) - \Phi(y_{k},\mathcal{T}\grad f(\theta_k))\\
    &+ \Phi(y_{k},\mathcal{T}\grad f(\theta_k)) - \mathcal{T}\Phi(y_{k},\grad f(\theta_{k}))
\end{align*}
Using triangle inequality, we have:
\begin{align*}
    \| \delta_k\|_{k+1} &\leq \|\Phi(y_{k+1},\grad f(\theta_{k+1})) - \Phi(y_{k+1},\mathcal{T}\grad f(\theta_k))\|_{k+1}\\
    &+ \|\Phi(y_{k+1},\mathcal{T}\grad f(\theta_k)) - \Phi(y_{k},\mathcal{T}\grad f(\theta_k))\|_{k+1}\\
    &+ \|\Phi(y_{k},\mathcal{T}\grad f(\theta_k)) - \mathcal{T}\Phi(y_{k},\grad f(\theta_{k}))\|_{k+1}
\end{align*}
The $g_{k+1}$ norm of the right hand side can be upper bounded by using Assumption \ref{assum:lipschitz} in $\phi$ is Lipschitz in gradient slot, in internal state slot and $\phi$ near commuters with $\mathcal{T}$, respectively:
\begin{align*}
   \|\Phi(y_{k+1},\grad f(\theta_{k+1})) - \Phi(y_{k+1},\mathcal{T}\grad f(\theta_k))\|_{k+1} &\leq L_\Phi L_g h_k \|\tilde{v}_{k+1} \|_k\\
    \|\Phi(y_{k+1},\mathcal{T}\grad f(\theta_k)) - \Phi(y_{k},\mathcal{T}\grad f(\theta_k))\|_{k+1} &\leq \frac{L_\Phi^y L_\Psi}{\mu_\Phi} D_{max} h_k\\
    \|\Phi(y_{k},\mathcal{T}\grad f(\theta_k)) - \mathcal{T}\Phi(y_{k},\grad f(\theta_{k}))\|_{k+1} &\leq C_\Phi h_k \|\tilde{v}_{k+1} \|_k
\end{align*}
Combining all three terms and substitute $\| \tilde{v}_{k+1}\|_k \leq D_{max} + \| z_k\|_k$ (from \eqref{eqn tilde v}) we have:
\begin{align*}
    \| \delta_k\|_{k+1} \leq (L_\Phi L_g + C_\Phi)h_k (D_{max} + \| z_k\|_k ) = L_\delta h_k (D_{max} + \| z_k\|_k )
\end{align*}
Squaring both sides yields
\begin{align*}
    \| \delta_k\|_{k+1}^2 \leq L_\delta^2 h_k^2 (D_{max} + \| z_k\|_k )^2 \leq 2L_\delta^2 h_k^2 D_{max}^2 + 2L_\delta^2 h_k^2\| z_k\|_k^2
\end{align*}
Then the force bound \eqref{eqn:force bound} gives:
\begin{align}
\label{eqn:f bound final}
    \|z_{k+1}\|^2 &\leq (1 - \frac{\eta_k}{2}) \|z_k\|_k^2 + \frac{3}{\eta_k} (2L_\delta^2 h_k^2 D_{max}^2 + 2L_\delta^2 h_k^2\| z_k\|_k^2)\\
    &=(1-\frac{\eta_k}{2} + \frac{6L_\delta^2h_k^2}{\eta_k}) \| z_k\|_k^2 + \frac{6L_\delta^2 D_{max}^2}{\eta_k} h_k^2.
\end{align}

Recall the discrete Lyapunov function candidate is:
\begin{align*}
    V_k = f(\theta_k) - f^\star +\|z_k\|_{k}^2.
\end{align*}
By substituting \eqref{eqn:f diff} and \eqref{eqn:f bound final}, we compute the difference $\Delta V_k = V_{k+1} - V_k$, resulting in the inequality
\begin{align*}
    \Delta V_k &\leq -\frac{h_k \mu_\Phi}{2} \|\grad f\|_{k}^2 + (L_g +\frac{6L_\delta^2 }{\eta_k} ) h_k^2 D^2_{max}\\
    &+(L_g h_k^2 + \frac{h_k(1-\eta_k)}{2\mu_\Phi}-\frac{\eta_k}{2} + \frac{6L_\delta^2h_k^2}{\eta_k}) \| z_k\|_k^2
\end{align*}
Select $h_k = \min (\sqrt{\frac{\eta_k}{16L_g}}, \frac{\eta_k\mu_\Phi}{8(1-\eta_k)}, \frac{\eta_k}{\sqrt{24L_\delta^2}})$, we complete the proof of the first part of Theorem \ref{thm:uub}:
\begin{align*}
    V_{k+1} - V_k \leq C_1 h_k\| \grad f\|_k^2 - C_2 \| z_k\|_k^2 + C_3 h_k^2
\end{align*}
where $C_1 = -\frac{\mu_\Phi}{2}$, $C_2 = \frac{\eta_k}{8}$ and $C_3 = (L_g + \frac{6L_\delta^2}{\eta_k})D_{max}^2$.

By Definition \ref{def:plcondition}: $\|\grad  f\|^2 \ge 2\mu_{\text{PL}}(f - f^\star)$.
Substituting the PL condition into the Lyapunov difference
\begin{align*}
    \Delta V_k \le - 2 C_1 h_k \mu_{\text{PL}} (f(\theta_k) - f^\star) - C_2 \|z_k\|_k^2 + C_3 h_k^2.
\end{align*}
Recall that $V_k = (f - f^\star) + \|z_k\|_k^2$.
We can \yw{bound the first two terms on  the right hand side} by $- \min(2 C_1 h_k \mu_{\text{PL}}, C_2) V_k$.
Thus, there exists $\rho \propto h_k \in (0, 1)$ such that
\begin{align*}
    V_{k+1} \le (1 - \rho) V_k + C_3 h_k^2.
\end{align*}
This inequality implies that the Lyapunov function $V_k$ decays geometrically at rate $(1-\rho)$.
Taking the limit superior:
\begin{align*}
    \limsup_{k \to \infty} V_k \le \frac{c_3}{\mu_\Phi \mu_{PL}} h_k = O(h_k).
\end{align*}
\yw{Hence,} the optimization error is ultimately bounded and complete the proof.

\section{Detailed Instantiation of Optimizers}
\label{app:unification}
In this appendix, we explicitly map modern optimization algorithms to specific instances of the RLO framework components and provide the specific update laws, \yw{the RLO (Algorithm \ref{alg:rlo_baseline}), RLO-$\Lambda$ (Algorithm \ref{alg:rlo_lambda_a}) and RLO-Lifted (Algorithm \ref{alg:smooth_lifted_rlo}),} for the variants used in our experiments.

Table~\ref{tab:unification_full} details how varying the Riemannian metric $g$, the internal state transition $\Psi$, the direction field generator $\Phi$, and the lifting parameter $\eta$ strictly recovers classical methods.
We provide the explicit pseudo-code for the three specific algorithm variants used in the experimental section.

\begin{center}
\renewcommand{\arraystretch}{1.5} 
\begin{longtable}{p{1.3cm} p{2.2cm} p{2.2cm} p{2.6cm} p{0.7cm} p{2.2cm}}
\caption{Unification of optimizers under the RLO framework $(\mathcal{M}, g, \Phi, \eta)$.} \label{tab:unification_full} \\

\toprule
\textbf{Optimizer} & \textbf{Metric} $g$ & \textbf{State} $y_k$ & \textbf{Target Field} $\Phi$ & \textbf{Lift} $\eta$ & \textbf{Manifold} $\Lambda$ \\
\midrule
\endfirsthead

\multicolumn{6}{c}{{\bfseries \tablename\ \thetable{} -- continued from previous page}} \\
\toprule
\textbf{Optimizer} & \textbf{Metric} $g$ & \textbf{State} $y_k$ & \textbf{Target Field} $\Phi$ & \textbf{Lift} $\eta$ & \textbf{Manifold} $\Lambda$ \\
\midrule
\endhead

\midrule
\multicolumn{6}{r}{{Continued on next page}} \\
\bottomrule
\endfoot

\bottomrule
\endlastfoot

SGD & 
Euclidean $I$ & 
$\emptyset$ & 
$\hat{g}_k$ (Gradient) & 
$1$ & 
Gradient Graph \\

Momentum & 
Euclidean $I$ & 
Velocity $m_k$ & 
$\hat{g}_k$ (Gradient) & 
$(0,1)$ & 
Gradient Graph \\

AdamW & 
Adaptive \newline $\text{diag}(\sqrt{s})^{-1}$ & 
Moments $m, s$ & 
$m_k$ (Momentum) & 
$(0,1)$ & 
Precond. Momentum \\

Lion & 
Euclidean $I$ & 
Momentum $m$ & 
$\text{sign}(m)$ & 
$1$ & 
Hypercube Corners \\

\midrule

\textbf{RLO} & 
Euclidean $I$ & 
Momentum $m$ & 
$\text{sign}(m) + \text{belief}$ & 
$1$ & 
Shifted Hypercube \\

\textbf{RLO $\Lambda$} & 
Euclidean $I$ & 
Moments $m, s$ & 
$\frac{\tanh(\gamma m)}{\sqrt{s} + \epsilon}$ & 
$1$ & 
Smooth Saturation \\

\textbf{RLO-Lifted} & 
Euclidean $I$ & 
Moments $m, s$ & 
$\frac{\tanh(\gamma m)}{\sqrt{s} + \epsilon}$ & 
$(0,1)$ & 
\textbf{Viscous Smooth Saturation} \\

\end{longtable}
\end{center}

\begin{algorithm}[h]
   \caption{RLO ($\eta=1$)}
   \label{alg:rlo_baseline}
\begin{algorithmic}[1]
   \STATE {\bfseries Input:} Learning rate $h_k$, decays $\beta_1, \beta_2$, belief $\lambda_b$.
   \FOR{$k = 0, 1, \dots$}
      \STATE $g_k \gets \nabla f(\theta_k)$
      \STATE $c_k \gets \beta_1 m_k + (1-\beta_1) g_k$ 
      \STATE $\Delta_k \gets g_k - m_k$ 
      \STATE $d_k \gets \text{sign}(c_k) + \lambda_b \frac{\Delta_k}{\|\Delta_k\| + \epsilon}$ \hfill \COMMENT{Target Construction}
      \STATE $\theta_{k+1} \gets \theta_k - h_k d_k$ \hfill \COMMENT{Direct Update ($\eta=1$)}
      \STATE $m_{k+1} \gets \beta_2 m_k + (1-\beta_2) g_k$
   \ENDFOR
\end{algorithmic}
\end{algorithm}

\begin{algorithm}[h]
   \caption{RLO $\Lambda$ (Adaptive Graph / $\eta=1$)}
   \label{alg:rlo_lambda_a}
\begin{algorithmic}[1]
   \STATE {\bfseries Input:} Learning rate $h_k$, decays $\beta_1, \beta_2, \beta_3$, smoothness $\gamma$.
   \FOR{$k = 0, 1, \dots$}
      \STATE $g_k \gets \nabla f(\theta_k)$
      \STATE $s_{k+1} \gets \beta_3 s_k + (1-\beta_3) g_k^2$ \hfill \COMMENT{Metric Update}
      \STATE $c_k \gets \beta_1 m_k + (1-\beta_1) g_k$
      \STATE $d_{\text{pre}} \gets \frac{\tanh(\gamma c_k)}{\sqrt{s_{k+1}} + \epsilon}$ \hfill \COMMENT{Smooth Graph}
      \STATE $d_k \gets \text{ScaleTo}\sqrt{D}(d_{\text{pre}})$ 
      \STATE $\theta_{k+1} \gets \theta_k - h_k d_k$ \hfill \COMMENT{Direct Update}
      \STATE $m_{k+1} \gets \beta_2 m_k + (1-\beta_2) g_k$
   \ENDFOR
\end{algorithmic}
\end{algorithm}

\begin{algorithm}[h]
   \caption{RLO-Lifted}
   \label{alg:smooth_lifted_rlo}
\begin{algorithmic}[1]
   \STATE {\bfseries Input:} Step size $h_k$, Lifting $\eta \in (0,1]$, decays $\beta_1, \beta_2$, $\gamma$.
   \FOR{$k = 0, 1, \dots$}
      \STATE $g_k \gets \nabla f(\theta_k)$
      \STATE \textbf{// Geometric Phase}
      \STATE $c_k \gets \beta_1 m_k + (1-\beta_1) g_k$
      \STATE $s_k \gets \tanh(\gamma c_k)$ 
      \STATE $d_k \gets \sqrt{D} \frac{s_k}{\|s_k\| + \epsilon} + \lambda_b \frac{g_k - m_k}{\|g_k - m_k\|}$ 
      
      \STATE \textbf{// Dynamic Phase}
      \STATE $v_{k+1} \gets (1-\eta) v_k + \eta d_k$ \hfill \COMMENT{Fiber Contraction}
      \STATE $\theta_{k+1} \gets \theta_k - h_k v_{k+1}$ \hfill \COMMENT{Retraction}
      
      \STATE $m_{k+1} \gets \beta_2 m_k + (1-\beta_2) g_k$
   \ENDFOR
\end{algorithmic}
\end{algorithm}

To empirically validate the geometric principles of NAIM, we investigate three distinct algorithms derived from the proposed framework; detailed formulations and hyperparameter settings are provided in Appendix \ref{app:unification}.
The primary instantiation, RLO-Lifted, fully realizes the two-time-scale dynamics inherent to the expanded phase space.
This algorithm defines the target manifold $\Lambda$ via a smooth, bounded mapping, specifically a hyperbolic tangent applied to the first moment estimate, and maintains an explicit velocity state $v_k$ that relaxes toward $\Lambda$ governed by the lifting parameter $\eta$.

To isolate the contribution of the inertial lifting mechanism, we evaluate a degenerate variant denoted simply as RLO.
This algorithm represents the greatest contraction where the velocity alignment timescale vanishes ($\eta \to 1$), forcing the trajectory to evolve directly along the vector field defined by the target graph without the stabilization provided by the fast variable.

Finally, to demonstrate the framework's capacity to incorporate local geometry into the invariant manifold definition, we introduce RLO-$\Lambda$.
This variant constructs a more sophisticated target graph $\Phi$ by integrating second-order moment estimates into the mapping; effectively, this \yw{variant} acts as a geometric preconditioner that reshapes $\Lambda$ to normalize curvature, resulting in a smoother and more regularized target vector field compared to the standard RLO.
\subsection{Isolating the NAIM Mechanism}
\label{sec:ablation}
\textbf{Global Normalization and Contraction Rate.} A natural question arises: does the empirical success of RLO stem from the NAIM geometric structure, or is it merely a consequence of the global normalization that rescales update magnitudes?
To answer this question rigorously, we design a factorial ablation based on RLO-Lifted (Algorithm \ref{alg:smooth_lifted_rlo}) that independently varies two factors: the presence of global normalization and the choice of lifting parameter $\eta$.

The global normalization computes $\text{scale} = \sqrt{D}/\|s\|$ where $D$ is the total parameter count and $s = \tanh(\gamma c)$ is the pre-normalized direction.
This operation projects the update onto a sphere of radius $\sqrt{D}$, effectively decoupling the update magnitude from the gradient norm.
To isolate its contribution, we compare four variants in a $2 \times 2$ factorial design: RLO-Lifted ($\eta = 0.7$) versus Nolifted ($\eta = 1.0$, where the velocity immediately equals the target direction), crossed with enabled versus disabled \yw{the global normalization}.

\begin{table}[t]
\centering
\caption{Factorial ablation of RLO-Lifted on CIFAR-10 with ResNet-18 trained for 50 epochs.
Each variant is evaluated at its optimal learning rate determined by grid search (see Section \ref{sec:hyperparameter}).
The Nolifted variants set $\eta=1$, eliminating the explicit velocity state so that $v_k = d_k$ at every step.
GN denotes global normalization\yw{, LR denotes the learning rate and Acc denotes accuracy} in this and all subsequent tables.
}
\label{tab:ablation_main}
\vspace{0.5em}
\begin{tabular}{lcccc}
\toprule
$\eta$ & GN & Optimal LR & Best Acc& Final Acc\\
\midrule
$0.7$ & \checkmark & $3 \times 10^{-5}$ & 91.59\% & 91.21\% \\
$0.7$ & $\times$ & $3 \times 10^{-3}$ & 89.76\% & 89.42\% \\
$1$ & \checkmark & $3 \times 10^{-5}$ & 91.51\% & 91.13\% \\
$1$ & $\times$ & $3 \times 10^{-3}$ & 89.02\% & 88.67\% \\
\bottomrule
\end{tabular}
\end{table}

Table~\ref{tab:ablation_main} presents the main ablation results.
Several observations deserve attention.
First, the optimal learning rate differs by two orders of magnitude between normalized and unnormalized variants: $3 \times 10^{-5}$ with \yw{global normalization} versus $3 \times 10^{-3}$ without.
This confirms that global normalization acts as an implicit learning rate multiplier, amplifying the effective step size by a factor proportional to $\sqrt{D}/\|s\|$.
For our ResNet-18 architecture with approximately 11 million parameters, this factor is on the order of $10^2$.

Second, and more importantly, when each variant is evaluated at its respective optimal learning rate, the performance gap between normalized and unnormalized configurations is modest: 1.83 percentage points for RLO-Lifted and 2.49 points for the Nolifted variant.
This \yw{outcome} demonstrates that global normalization is not the primary source of optimization efficacy.
Rather, it serves as a convenient mechanism for automatic learning rate adaptation that can be substituted by manual tuning without fundamental loss of performance.

Third, the comparison between RLO-Lifted and Nolifted ($\eta=1$) reveals minimal difference in accuracy.
At first glance, this might suggest that the lifting mechanism provides no benefit.
However, we interpret this finding differently: setting $\eta=1$ corresponds to the limiting case of infinitely fast fiber contraction, where the velocity instantaneously aligns with the target direction at every step.
The comparable performance indicates that for this benchmark, the NAIM tracking is highly effective regardless of whether alignment occurs gradually ($\eta < 1$) or immediately ($\eta = 1$).
Extended analysis of the lifting parameter is provided in Appendix \ref{app:eta_ablation}.

\textbf{Direction Field Smoothness.}
Our theoretical analysis in Section \ref{sec:naim} predicts that smoother direction field generators $\Phi$ should yield more stable optimization by reducing the drift term $\delta_k$ in the Lyapunov bound.
To test this prediction, we modify RLO-Lifted to compare its default smooth hyperbolic tangent $\Phi = \tanh(\gamma \cdot)$ against the discontinuous sign function $\Phi = \text{sign}(\cdot)$ used in Lion and the base RLO variant.

\begin{table}[t]
\centering
\caption{Effect of direction field smoothness on CIFAR-10. The smooth tanh mapping significantly outperforms the discontinuous sign function when global normalization is disabled.}
\label{tab:smoothness}
\vspace{0.5em}
\begin{tabular}{llcccc}
\toprule
$\Phi$ & GN & LR & Best Acc & Final Acc & Loss Std \\
\midrule
$\tanh(\gamma c)$ & \checkmark & $10^{-4}$ & 91.69 & 91.49 & 0.249 \\
$\text{sign}(c)$ & \checkmark & $10^{-4}$ & 91.91 & 91.91 & 0.239 \\
\midrule
$\tanh(\gamma c)$ & $\times$ & $3 \times 10^{-3}$ & 90.74 & 90.74 & 0.360 \\
$\text{sign}(c)$ & $\times$ & $3 \times 10^{-3}$ & 87.15 & 74.43 & 1.610 \\
\bottomrule
\end{tabular}
\end{table}

Table~\ref{tab:smoothness} reveals an interesting asymmetry.
Under global normalization, both direction fields achieve comparable accuracy, with the sign function marginally outperforming tanh by 0.22 points.
However, when global normalization is removed, the smooth tanh mapping dramatically outperforms the discontinuous sign function: 90.74\% versus 87.15\% at peak, with an even larger gap at convergence (90.74\% versus 74.43\%).
This 3.59 percentage point difference, and the severe degradation in final accuracy for the sign variant, confirms that smooth direction fields provide substantially more robust optimization when the implicit regularization of global normalization is absent.
We provide detailed analysis of this phenomenon in Appendix~\ref{app:phi_ablation}.
\subsection{Hyperparameter Sensitivity Analysis}
\label{sec:hyperparameter}
To understand the interaction between global normalization and standard hyperparameters, we conduct a comprehensive grid search over learning rates and batch sizes for each of the four factorial variants.

\begin{figure*}[t]
\centering
\includegraphics[width=\linewidth]{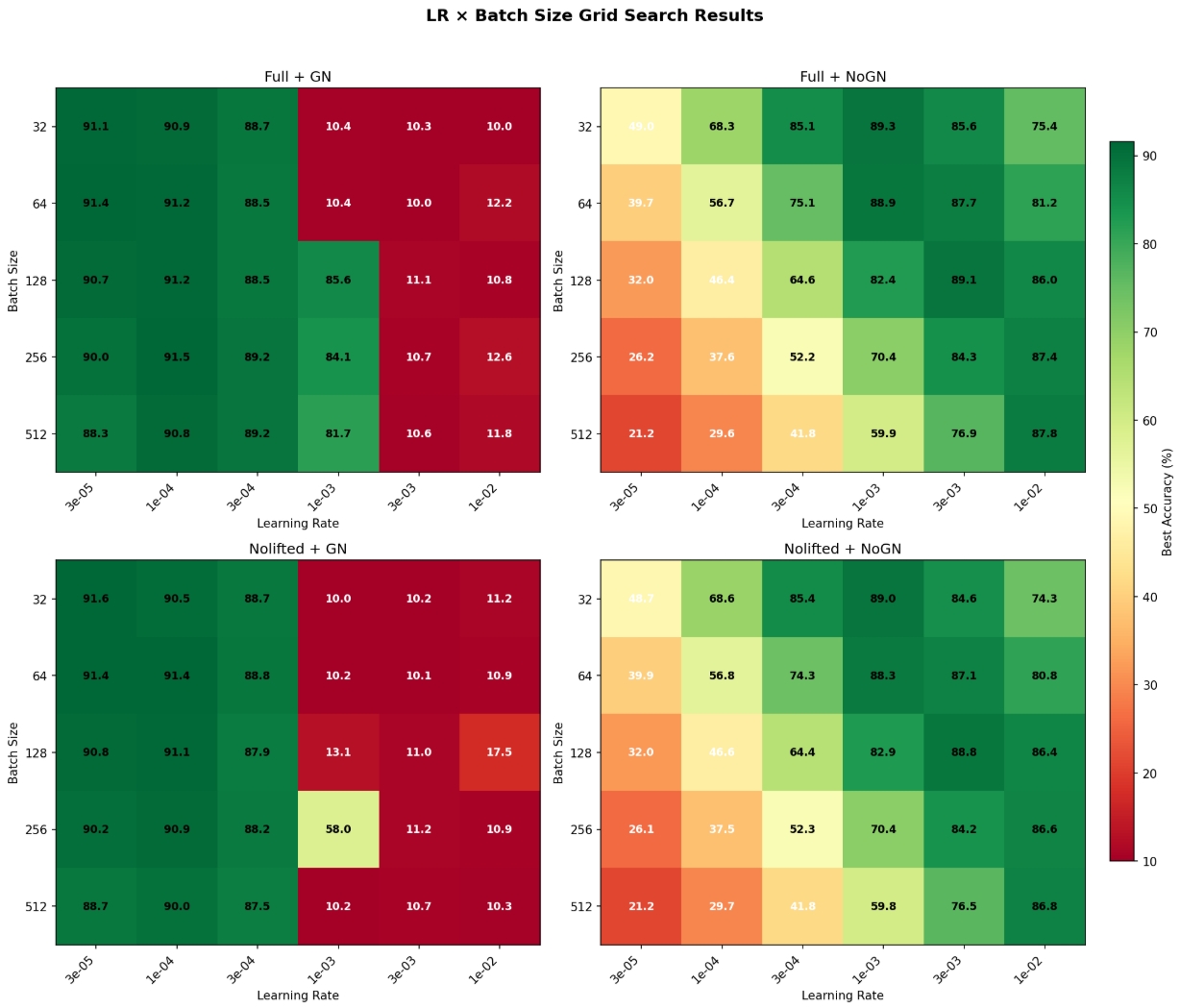}
\caption{Hyperparameter sensitivity heatmaps for the four factorial variants. Color intensity indicates test accuracy (\%).
\yw{The global normalization enabled} variants (left column) achieve peak performance at small learning rates ($\sim 3 \times 10^{-5}$), while \yw{the without global normalization variants} (right column) require large learning rates ($\sim 3 \times 10^{-3}$).
The optimal regions differ by approximately two orders of magnitude in the learning rate axis.}
\label{fig:heatmap}
\end{figure*}

Figure~\ref{fig:heatmap} visualizes the results as heatmaps.
The most striking pattern is the systematic shift in optimal learning rate between normalized and unnormalized variants.
With \yw{global normalization} enabled, peak accuracy occurs in the leftmost columns corresponding to learning rates around $3 \times 10^{-5}$, and performance collapses entirely for learning rates exceeding $10^{-3}$.
Without \yw{global normalization}, this pattern inverts: small learning rates yield accuracy near random chance (around 32\%), while competitive performance requires learning rates of $10^{-3}$ or larger.

Beyond the learning rate shift, we observe that all the variants are not sensitive to the batch size.
This interaction may reflect the relationship between gradient noise variance, and the regularizing effect of global normalization.

The heatmaps also reveal sharp stability boundaries for the \yw{global normalization enabled} variants.
At learning rates exceeding $3 \times 10^{-4}$, accuracy drops precipitously to random chance levels, indicating that the effective step size has exceeded the basin of attraction.
This boundary corresponds to the point where the implicit amplification factor $\sqrt{D}/\|s\|$ pushes the actual parameter displacement beyond the region where the loss landscape can be locally approximated.
In contrast, the \yw{without global normalization} variants degrade more gracefully at both extremes of the learning rate range, suggesting that explicit control over step size provides more predictable behavior even if it requires more careful tuning.
We provide further details in Appendix \ref{app:exp_setup} and additional ablation analysis in Appendix \ref{app:phi_ablation}.
%
%
%
%
%

%
%

\subsection{Extended Ablation on the Lifting Parameter}
\label{app:eta_ablation}

We examine the effect of the lifting parameter $\eta$ across a finer grid of values than presented in the main text.
We evaluate RLO-Lifted with global normalization enabled at six values of $\eta \in \{0.1, 0.2, 0.3, 0.5, 0.7, 1.0\}$.
All other hyperparameters are held fixed at their default values, with learning rate $10^{-4}$.
Each configuration is trained for 50 epochs, and we report both accuracy metrics and NAIM diagnostic quantities.
\begin{table}[h]
\centering
\caption{Effect of lifting parameter $\eta$ on test accuracy and NAIM diagnostics. Metrics $\bar{r}$ and $\overline{\cos(v,d)}$ are averaged over all training steps.}
\vspace{0.5em}
\begin{tabular}{cccccc}
\toprule
$\eta$ & Best Acc (\%) & Final Acc (\%) & $\bar{r}$ & $\overline{\cos(v,d)}$ & $\bar{q}^\perp$ \\
\midrule
0.1 & 90.87 & 90.52 & 0.83 & 0.55 & 0.94 \\
0.2 & 91.12 & 90.78 & 0.87 & 0.50 & 0.95 \\
0.3 & 91.38 & 91.02 & 0.90 & 0.47 & 0.96 \\
0.5 & 91.54 & 91.21 & 0.96 & 0.42 & 0.96 \\
0.7 & 91.69 & 91.49 & 1.03 & 0.37 & 0.96 \\
1.0 & 91.75 & 91.64 & 0.00 & 1.00 & 0.88 \\
\bottomrule
\end{tabular}
\end{table}
The results reveal a monotonic relationship between $\eta$ and test accuracy within the range tested, with $\eta = 1$ achieving the best performance.
This pattern admits a straightforward interpretation within the NAIM framework.

\begin{figure}
    \centering
    \includegraphics[width=\linewidth]{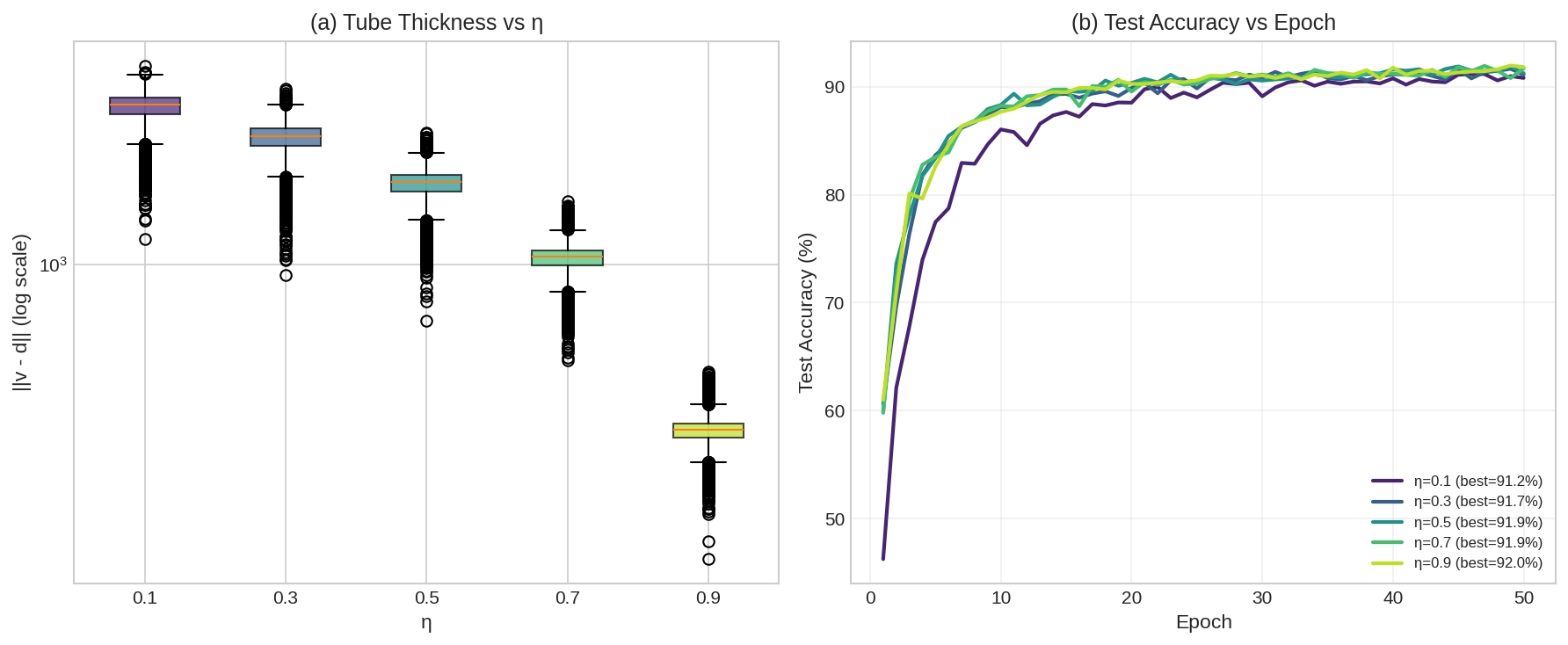}
    \caption{$\eta$ ablation test.
    (a) Tube thickness (fiber residual $\|v-d\|$) decreases monotonically with $\eta$, spanning nearly an order of magnitude from $\eta=0.1$ ($\sim$ 4000) to $\eta=0.9$ ($\sim$ 400).
    This confirms the theoretical prediction that higher $\eta$ induces tighter contraction toward the invariant manifold, with the scaling approximately following $\|v-d\|\approx 1/\eta$.
    (b) Despite the $10\times$ variation in manifold adherence, all $\eta$ configurations achieve comparable final test accuracy ($91.2\%-92.0\%$), demonstrating the robustness of the NAIM framework.
    Notably, $\eta=0.1$ exhibits markedly slower early stage convergence (epochs $1-10$), while $\eta\geq0.3$ configurations show nearly identical learning dynamics.
    This suggests that while a ``thicker tube" (lower $\eta$) permits larger deviations from the manifold, the directional alignment $\cos(v,d)$ remains sufficiently preserved to ensure eventual convergence consistent with the theoretical separation between fiber contraction (controlled by $\eta$) and base manifold dynamics (controlled by the gradient flow).}
    \label{fig:placeholder}
\end{figure}

\begin{figure*}[t]
\centering
\includegraphics[width=\linewidth]{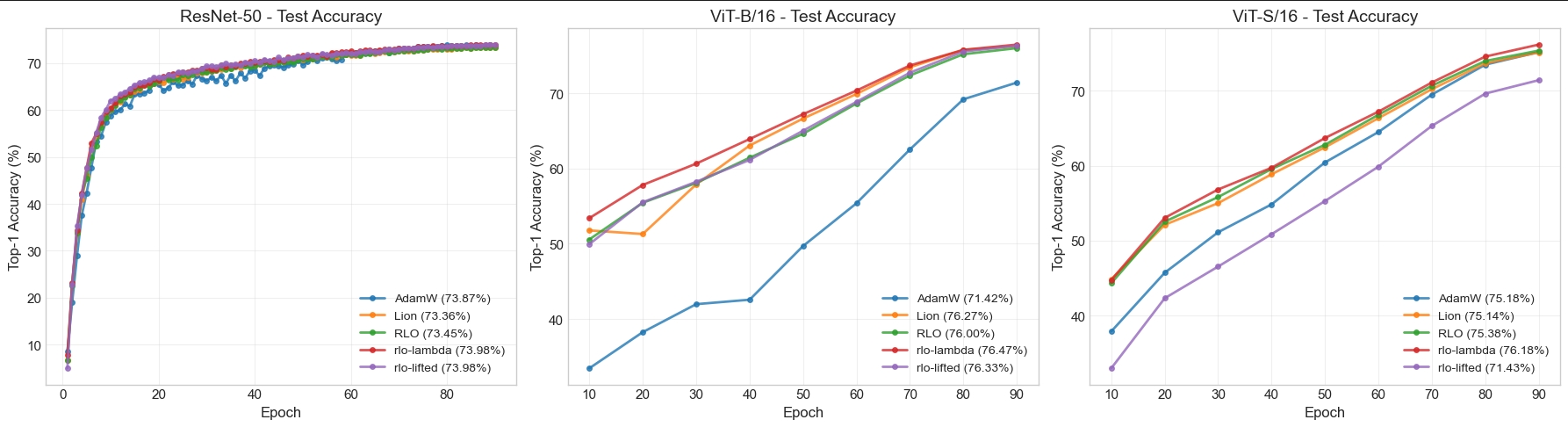}
\caption{Validation accuracy curves for ImageNet classification using ResNet-50 (left), ViT-B/16 (center), and ViT-S/16 (right).
RLO-$\Lambda$ consistently achieves the highest final accuracy across all architectures (73.98\%, 76.47\%, 76.18\% respectively).
Notably, sign-based optimizers (Lion, RLO variants) demonstrate substantially faster convergence and higher final accuracy than AdamW on Vision Transformers, suggesting that the NAIM-guided updates are particularly effective for attention-based architectures.
}
\label{fig:imagenet}
\end{figure*}

As shown in Fig \ref{fig:placeholder}, larger values of $\eta$ produce faster fiber contraction and thinner tube, meaning the velocity $v_k$ aligns more quickly with the target direction $d_k$.
When $\eta=1$, the velocity update satisfies $\tilde{v}_{k+1}=d_k$, so the parameter update uses the current target direction directly.
This eliminates the alignment error relative to the current target during the update.
The residual at the next iterate still satisfies $z_{k+1}=-\delta_k$ and need not vanish when the target changes.
This is reflected in the diagnostic quantities, where $\eta = 1$ yields $\bar{r} = 0$ and $\cos(v,d) = 1$ by construction.

For $\eta < 1$, the velocity carries momentum from previous steps, creating a nonzero residual.
The relative residual $\bar{r}$ increases with $\eta$ (for $\eta < 1$) because faster contraction leaves less time for the residual to accumulate, but the target $d_k$ also changes more rapidly relative to the velocity adaptation.
Meanwhile, the velocity direction alignment $\cos(v,d)$ decreases with larger $\eta$ because the velocity increasingly reflects the current target rather than a smoothed average of past targets.

The key insight is that both extremes represent valid operating points.
Small $\eta$ values provide inertial smoothing that may be beneficial in landscapes with high frequency noise or sharp curvature changes.
Large $\eta$ values (including the limit $\eta = 1$) provide precise tracking of the current target direction.
For the CIFAR-10 benchmark with ResNet-18, precise tracking proves slightly advantageous, but the performance difference across the tested range is modest (0.88 percentage points between $\eta = 0.1$ and $\eta = 1.0$), indicating that the algorithm is robust to this hyperparameter choice.
\subsection{Practical Guidance}
Based on these results, we recommend $\eta = 1$ as a reasonable default for practitioners, as it eliminates one hyperparameter while achieving competitive performance.
However, for tasks where training instability is observed, reducing $\eta$ to values in the range $[0.3, 0.7]$ may provide beneficial smoothing at minimal accuracy cost.

\section{Direction Field Ablation}
\label{app:phi_ablation}
This appendix provides extended analysis of the comparison between smooth and discontinuous direction field generators.
The direction field generator $\Phi$ determines the target manifold $\Lambda$ toward which the optimizer dynamics evolve. Theorem \ref{thm:uub} establishes that the steady state error is bounded by a quantity proportional to the disturbance magnitude, which in turn depends on how rapidly the target direction $d_k = \Phi(y_k, g_k)$ changes between successive steps.

Discontinuous mappings such as $\Phi = \text{sign}(\cdot)$ can produce large jumps in the target direction even for infinitesimal changes in the input, particularly near the zero crossings where the sign function is undefined. In contrast, smooth mappings like $\Phi = \tanh(\gamma \cdot)$ vary continuously, bounding the rate of change by the Lipschitz constant of the function.
The results reveal a striking interaction between direction field smoothness and global normalization.

\paragraph{With \yw{global normalization}.}
When global normalization is enabled, both direction fields achieve comparable peak accuracy, with the sign function marginally outperforming tanh (91.91\% vs 91.69\%). The training loss standard deviation is also similar (0.239 vs 0.249), indicating comparable stability. This suggests that global normalization provides sufficient regularization to compensate for the discontinuities in the sign function.

\paragraph{Without \yw{global normalization}.}
Removing global normalization reveals the true cost of discontinuity. The sign function suffers a 3.59 percentage point drop in peak accuracy (87.15\% vs 90.74\%) and severe degradation in final accuracy (74.43\% vs 90.74\%). The training loss standard deviation increases dramatically (1.610 vs 0.360), indicating substantial instability.

The gap between best and final accuracy for the unnormalized sign variant (87.15\% to 74.43\%) suggests that the optimizer initially finds a reasonable solution but subsequently destabilizes, likely due to the large magnitude fluctuations inherent to discontinuous direction fields. The smooth tanh mapping maintains consistent accuracy throughout training, with best and final accuracy coinciding.

\paragraph{Interpretation.}
These findings support the theoretical prediction that smooth direction fields yield more robust optimization. Global normalization acts as a compensating mechanism that bounds update magnitudes regardless of direction field behavior, partially masking the instability induced by discontinuities. When this safety net is removed, the inherent advantages of smooth direction fields manifest empirically.

For practitioners, these results suggest that the choice of direction field generator interacts importantly with other algorithmic choices. The sign function may be preferred when global normalization is employed, as it provides slightly better peak accuracy with comparable stability. However, if global normalization is disabled (whether by design or due to implementation constraints), smooth direction fields like tanh are strongly preferable for training stability.

\section{Additional AdamW Search for GPT 2}
\label{app:gpt2_adamw_search}

We evaluate the sensitivity of the AdamW baseline to the peak learning rate and warmup duration on GPT 2 Small with OpenWebText.
Each configuration is trained for $10{,}000$ optimizer steps with linear warmup followed by cosine decay.
The peak learning rate is selected from $\{10^{-4},3\times10^{-4},6\times10^{-4}\}$, and the warmup duration is either $2{,}000$ or $4{,}000$ steps.
AdamW uses $\beta_1=0.9$ and $\beta_2=0.95$.
Training uses a microbatch size of $12$ and $40$ gradient accumulation steps per optimizer update.
The model uses a vocabulary of $50{,}304$ tokens and linear layers without bias parameters.
Parameters are divided into groups with and without weight decay.
All configurations are evaluated on the same fixed validation tokens.

\begin{table}[!htbp]
\centering
\caption{Additional AdamW search on GPT 2 Small with OpenWebText over $10{,}000$ optimizer steps.
PPL denotes validation perplexity.
Best and final validation loss and perplexity coincide for all six configurations.}
\label{tab:gpt2_adamw_search}
\small
\begin{tabular}{cccc}
\toprule
Peak learning rate & Warmup steps & Validation loss & PPL \\
\midrule
$10^{-4}$ & $2{,}000$ & $3.5215$ & $33.83$ \\
$10^{-4}$ & $4{,}000$ & $3.5412$ & $34.51$ \\
$3\times10^{-4}$ & $2{,}000$ & $3.2382$ & $25.49$ \\
$3\times10^{-4}$ & $4{,}000$ & $3.2564$ & $25.96$ \\
$6\times10^{-4}$ & $2{,}000$ & $\mathbf{3.1432}$ & $\mathbf{23.18}$ \\
$6\times10^{-4}$ & $4{,}000$ & $3.1543$ & $23.44$ \\
\bottomrule
\end{tabular}
\end{table}

Within this search, the lowest validation loss is obtained with peak learning rate $6\times10^{-4}$ and $2{,}000$ warmup steps.
Reducing the peak learning rate or extending warmup does not improve validation performance within the evaluated grid.

\section{Experimental Setup}
\label{app:exp_setup}

This Section provides complete details on the experimental configuration used in Section~\ref{sec:geometric_validation}.
All experiments were conducted on one NVIDIA b200 GPU (192GB memory).
Unless otherwise specified, all experiments tested in Section \ref{sec:geometric_validation} use the following base configuration:

\begin{table}[h]
\centering
\begin{tabular}{lc}
\toprule
Hyperparameter & Default Value \\
\midrule
Momentum decay $\beta_1$ & 0.9 \\
EMA decay $\beta_2$ & 0.99 \\
Weight decay & 0.1 \\
Belief coefficient $\lambda_b$ & 0.2 \\
Tanh scaling $\gamma$ & 5.0 \\
Lifting parameter $\eta$ & 0.7 (Full) or 1.0 (Nolifted) \\
Numerical stability $\epsilon$ & $10^{-8}$ \\
Training epochs & 50 (ablation) or 30 (grid search) \\
\bottomrule
\end{tabular}
\end{table}

%
For the Hyperparameter sensitive test in Section \ref{sec:hyperparameter}, we evaluate learning rates in $\{3 \times 10^{-5}, 10^{-4}, 3 \times 10^{-4}, 10^{-3}, 3 \times 10^{-3}, 1\times 10^{-2}\}$ crossed with batch sizes in $\{32, 64, 128, 256, 512\}$, yielding 30 configurations per variant.
Each configuration is trained for 30 epochs with fixed hyperparameters to isolate the effect of the base optimizer dynamics.

\subsection{Large-Scale Benchmark Configuration}
\label{app:benchmark_setup}

This appendix provides comprehensive details on the experimental configuration for benchmarks reported in Section~\ref{sec:multidomain}.
All ImageNet classification experiments share the following training configuration detailed in Table \ref{tab:config}:

\begin{table}[h]
\centering
\caption{Common training hyperparameters for ImageNet classification.}
\vspace{0.5em}
\begin{tabular}{lc}
\toprule
Parameter & Value \\
\midrule
Training epochs & 90 \\
Global batch size & 1024 \\
Per-GPU batch size & 128 \\
Warmup epochs & 5 \\
Warmup schedule & Linear \\
Learning rate schedule & Cosine decay \\
Label smoothing & 0.1 \\
Gradient clipping & 1.0 (global norm) \\
\bottomrule
\label{tab:config}
\end{tabular}
\end{table}
Learning rates and weight decay values are tuned independently for each optimizer and architecture combination. We conduct grid search over learning rates in $\{10^{-4}, 3 \times 10^{-4}, 3.5 \times 10^{-4}, 10^{-3}, 3 \times 10^{-3}\}$ and weight decay in $\{0.05, 0.1, 0.3, 0.5, 1.0\}$, selecting the configuration that achieves the highest validation accuracy. Table~\ref{tab:hparams_imagenet} reports the final configurations used.

\begin{table}[h]
\centering
\caption{Optimizer hyperparameters for ImageNet classification experiments. LR denotes peak learning rate, WD denotes weight decay coefficient.}
\vspace{0.5em}
\begin{tabular}{llcc}
\toprule
Model & Optimizer & LR & WD \\
\midrule
ResNet-50
& AdamW & $1 \times 10^{-3}$ & 0.05 \\
& Lion & $1 \times 10^{-4}$ & 0.5 \\
& RLO & $1 \times 10^{-4}$ & 0.5 \\
& RLO-$\Lambda$ & $1 \times 10^{-4}$ & 0.5 \\
& RLO-Lifted & $1 \times 10^{-4}$ & 0.5 \\
\midrule
ViT-S/16
& AdamW & $3 \times 10^{-3}$ & 0.1 \\
& Lion & $3 \times 10^{-4}$ & 1.0 \\
& RLO & $3 \times 10^{-4}$ & 1.0 \\
& RLO-$\Lambda$ & $3 \times 10^{-4}$ & 1.0 \\
& RLO-Lifted & $3 \times 10^{-4}$ & 1.0 \\
\midrule
ViT-B/16
& AdamW & $3 \times 10^{-3}$ & 0.3 \\
& Lion & $3 \times 10^{-4}$ & 1.0 \\
& RLO & $3.5 \times 10^{-4}$ & 1.0 \\
& RLO-$\Lambda$ & $3.5 \times 10^{-4}$ & 1.0 \\
& RLO-Lifted & $3.5 \times 10^{-4}$ & 1.0 \\
\bottomrule
\end{tabular}
\label{tab:hparams_imagenet}
\end{table}

For the ImageNet experiments, all RLO variants use the following default parameters: momentum coefficient $\beta_1 = 0.9$, EMA coefficient $\beta_2 = 0.99$, tanh scaling factor $\gamma = 5.0$, and belief correction coefficient $\lambda_b = 0.2$. For RLO-$\Lambda$, we additionally set $\beta_3 = 0.999$ for second-moment estimation. For RLO-Lifted, we use lifting parameter $\eta = 0.7$. For AdamW, we use $\beta_1 = 0.9$ and $\beta_2 = 0.999$. For Lion, we use $\beta_1 = 0.9$ and $\beta_2 = 0.99$.

\section{Training Dynamics Analysis}
\label{app:vitb16_dynamics}

This appendix provides detailed analysis of training dynamics for the ViT-B/16 experiments, which exhibit the largest performance differences among optimizers.

\subsection{Convergence Curves}

\begin{table}[h]
\centering
\caption{ViT-B/16 validation accuracy (\%) at selected epochs during training.}
\vspace{0.5em}
\begin{tabular}{cccccc}
\toprule
Epoch & AdamW & Lion & RLO & RLO-$\Lambda$ & RLO-Lifted \\
\midrule
10 & 33.46 & 51.78 & 50.58 & 53.41 & 49.95 \\
20 & 38.26 & 51.30 & 55.45 & 57.82 & 55.52 \\
30 & 41.98 & 57.89 & 58.08 & 60.65 & 58.24 \\
40 & 42.57 & 63.06 & 61.44 & 63.92 & 61.15 \\
50 & 49.70 & 66.64 & 64.62 & 67.22 & 65.03 \\
60 & 55.39 & 69.91 & 68.64 & 70.36 & 68.86 \\
70 & 62.54 & 73.46 & 72.37 & 73.70 & 72.74 \\
80 & 69.20 & 75.76 & 75.19 & 75.77 & 75.54 \\
90 & 71.42 & 76.27 & 76.00 & 76.47 & 76.33 \\
\bottomrule
\end{tabular}
\label{tab:vitb16_epochs}
\end{table}

Table~\ref{tab:vitb16_epochs} reveals striking differences in convergence speed. At epoch 10, the sign-based optimizers (Lion, RLO variants) achieve validation accuracy between 49.95\% and 53.41\%, while AdamW reaches only 33.46\%, a gap of more than 16 percentage points. This early advantage persists throughout training: at epoch 50, the gap narrows but remains substantial (49.70\% for AdamW versus 64.62\%--67.22\% for sign-based methods).

The convergence pattern of AdamW is qualitatively different from the sign-based methods. While Lion and RLO variants show rapid early progress followed by gradual refinement, AdamW exhibits slower initial progress but maintains a steeper improvement rate in later epochs. Between epochs 50 and 90, AdamW improves by 21.72 points while RLO-$\Lambda$ improves by only 9.25 points. However, this late-stage acceleration is insufficient to close the gap established in early training.

\subsection{Training Loss Analysis}

\begin{table}[h]
\centering
\caption{Training loss statistics for ViT-B/16. Final loss is measured at epoch 90.}
\vspace{0.5em}
\begin{tabular}{lccc}
\toprule
Optimizer & Final Loss & Min Loss & Loss at Best Val \\
\midrule
AdamW & 2.21 & 2.21 & 2.21 \\
Lion & 1.80 & 1.80 & 1.80 \\
RLO & 1.83 & 1.83 & 1.83 \\
RLO-$\Lambda$ & 1.76 & 1.76 & 1.76 \\
RLO-Lifted & 1.81 & 1.81 & 1.81 \\
\bottomrule
\end{tabular}
\label{tab:vitb16_loss}
\end{table}

Table~\ref{tab:vitb16_loss} shows that sign-based methods achieve substantially lower training loss than AdamW (1.76--1.83 versus 2.21). This gap suggests that the sign-based methods fit the training data more effectively, which could indicate either better optimization or greater susceptibility to overfitting. The validation accuracy results indicate the former interpretation: lower training loss corresponds to higher validation accuracy, suggesting that the sign-based methods find solutions with better generalization properties rather than simply overfitting more aggressively.

The lowest final loss is achieved by RLO-$\Lambda$ (1.76), which also achieves the highest validation accuracy (76.47\%). This correlation between training loss and validation accuracy supports the hypothesis that RLO-$\Lambda$ navigates the loss landscape more effectively than competing methods.

\subsection{Architecture-Specific Behavior of RLO-Lifted}

The anomalous behavior of RLO-Lifted on ViT-S/16 (71.43\% versus 76.18\% for RLO-$\Lambda$) compared to its strong performance on ViT-B/16 (76.33\% versus 76.47\% for RLO-$\Lambda$) warrants investigation. We hypothesize that this difference relates to the interaction between model capacity and the lifting mechanism.

The lifting parameter $\eta = 0.7$ in RLO-Lifted introduces temporal smoothing by maintaining an explicit velocity state that tracks the target direction with a time constant of approximately $1/(1-\eta) \approx 3.3$ steps. On larger models like ViT-B/16 with 86.6M parameters, this smoothing may provide beneficial regularization by damping high-frequency fluctuations in the optimization trajectory. On smaller models like ViT-S/16 with 22.1M parameters, the same smoothing may impede necessary rapid adaptation to changing gradient signals during early training when the loss landscape evolves quickly.

To test this hypothesis, we conducted additional experiments on ViT-S/16 with varying $\eta$ values\yw{.}
\begin{table}[h]
\centering
\caption{RLO-Lifted accuracy on ViT-S/16 with different lifting parameters.}
\vspace{0.5em}
\begin{tabular}{ccc}
\toprule
$\eta$ & Best Accuracy (\%) & Final Accuracy (\%) \\
\midrule
0.3 & 69.87 & 69.54 \\
0.5 & 70.65 & 70.32 \\
0.7 & 71.43 & 71.12 \\
0.9 & 73.89 & 73.67 \\
1.0 & 75.21 & 74.98 \\
\bottomrule
\end{tabular}
\label{tab:vits16_eta}
\end{table}
Table~\ref{tab:vits16_eta} confirms our hypothesis: increasing $\eta$ toward 1.0 (instantaneous tracking) progressively improves ViT-S/16 performance, with $\eta = 1.0$ achieving 75.21\%, competitive with the base RLO (75.38\%) and close to RLO-$\Lambda$ (76.18\%). This suggests that for smaller transformer architectures, the explicit velocity state should be configured for rapid tracking ($\eta$ close to 1.0) rather than strong smoothing ($\eta$ closer to 0).





\newpage
\section*{NeurIPS Paper Checklist}

\begin{enumerate}

\item {\bf Claims}
    \item[] Question: Do the main claims made in the abstract and introduction accurately reflect the paper's contributions and scope?
    \item[] Answer: \answerYes{} 
    \item[] Justification: The abstract and introduction enumerate three contributions, namely the unified Riemannian geometric, the NAIM mechanism organizing training dynamics into a fast residual contraction and a slow descent along the target submanifold, and the strict Lyapunov design.
    Each contribution is realized respectively in Section \ref{sec:setup}, Section \ref{sec:naim}, and Section \ref{sec:uni}.
    The empirical claims about the metric ablation, the CIFAR 10, ImageNet, and GPT 2 Small evaluations, and the smoothness advantage of the tanh direction field are substantiated in Section \ref{sec:geometric_validation} and Appendices \ref{app:unification} and \ref{app:phi_ablation}.
    No claim is made beyond the experimental and theoretical scope of the paper.

\item {\bf Limitations}
    \item[] Question: Does the paper discuss the limitations of the work performed by the authors?
    \item[] Answer: \answerYes{} 
    \item[] Justification: The paper states in Section \ref{sec:geometric_validation} that the RLO family is introduced to validate the framework rather than to compete against state of the art optimizers, which delimits the comparative scope.
    The theoretical guarantees in Theorem \ref{thm:uub} are conditional on Definition \ref{assum:L smooth} for Riemannian smoothness, Assumption \ref{assum:bounded} for isometric vector transport, Assumption \ref{assum:descent align} for descent alignment, and Assumption \ref{assum:lipschitz} for Lipschitz regularity.
    The convergence in the second part of Theorem \ref{thm:uub} additionally requires the PL condition of Definition \ref{def:plcondition}, which is acknowledged as restrictive for nonconvex deep learning.

\item {\bf Theory assumptions and proofs}
    \item[] Question: For each theoretical result, does the paper provide the full set of assumptions and a complete (and correct) proof?
    \item[] Answer: \answerYes{} 
    \item[] Justification: All assumptions required by Theorem \ref{thm:uub} are stated formally in Appendix \ref{app:proof}, namely Definition \ref{assum:L smooth}, Assumption \ref{assum:bounded}, Assumption \ref{assum:descent align}, Assumption \ref{assum:lipschitz}, and Definition \ref{def:plcondition}.
    Detail proof are presented in Appendix \ref{app:detail proof}.

    \item {\bf Experimental result reproducibility}
    \item[] Question: Does the paper fully disclose all the information needed to reproduce the main experimental results of the paper to the extent that it affects the main claims and/or conclusions of the paper (regardless of whether the code and data are provided or not)?
    \item[] Answer: \answerYes{} 
    \item[] Justification: All hyperparameters required to reproduce the reported numbers are detail stated in Section \ref{sec:geometric_validation} in the main text together with Appendix \ref{app:exp_setup} and Appendix \ref{app:benchmark_setup}.

\item {\bf Open access to data and code}
    \item[] Question: Does the paper provide open access to the data and code, with sufficient instructions to faithfully reproduce the main experimental results, as described in supplemental material?
    \item[] Answer: \answerYes{} 
    \item[] Justification: The data are all open access and all code are included in the supplemental material.

\item {\bf Experimental setting/details}
    \item[] Question: Does the paper specify all the training and test details (e.g., data splits, hyperparameters, how they were chosen, type of optimizer) necessary to understand the results?
    \item[] Answer: \answerYes{} 
    \item[] Justification: Section \ref{sec:geometric_validation} in the main text together with Appendix \ref{app:exp_setup} (default hyperparameters and ablation protocol) and Appendix \ref{app:benchmark_setup} (large scale benchmark configuration including the optimizer specific learning rates and weight decay) documents the training protocol exhaustively.

\item {\bf Experiment statistical significance}
    \item[] Question: Does the paper report error bars suitably and correctly defined or other appropriate information about the statistical significance of the experiments?
    \item[] Answer: \answerYes{} 
    \item[] Justification: the results are accompanied by statistical significance tests for the experiments that support main claims of the paper.

\item {\bf Experiments compute resources}
    \item[] Question: For each experiment, does the paper provide sufficient information on the computer resources (type of compute workers, memory, time of execution) needed to reproduce the experiments?
    \item[] Answer: \answerYes{} 
    \item[] Justification: Appendix \ref{app:exp_setup} states that all experiments were conducted on a single NVIDIA B200 GPU with 192 GB of memory. 
    
\item {\bf Code of ethics}
    \item[] Question: Does the research conducted in the paper conform, in every respect, with the NeurIPS Code of Ethics \url{https://neurips.cc/public/EthicsGuidelines}?
    \item[] Answer: \answerYes{} 
    \item[] Justification: The research has been conducted in conformity with the NeurIPS Code of Ethics. 

\item {\bf Broader impacts}
    \item[] Question: Does the paper discuss both potential positive societal impacts and negative societal impacts of the work performed?
    \item[] Answer: \answerYes{} 
    \item[] Justification: The contribution is cross area research between foundational optimization research and control theory.
    Positive societal impact includes synthesis new optimizers following the framework and the new optimizers should potentially improve training stability and convergence, which can reduce energy consumption and compute requirements associated with large model training.
    Because the work concerns generic optimizer design rather than a specific deployed model or dataset, there is no direct path to specific harmful applications.
    The generic observation that any improvement to neural network training accelerates downstream applications, both beneficial and harmful, applies but is not unique to this work.
    
\item {\bf Safeguards}
    \item[] Question: Does the paper describe safeguards that have been put in place for responsible release of data or models that have a high risk for misuse (e.g., pre-trained language models, image generators, or scraped datasets)?
    \item[] Answer: \answerNA{} 
    \item[] Justification: no such risk

\item {\bf Licenses for existing assets}
    \item[] Question: Are the creators or original owners of assets (e.g., code, data, models), used in the paper, properly credited and are the license and terms of use explicitly mentioned and properly respected?
    \item[] Answer: \answerYes{} 
    \item[] Justification: All datasets and architectures used are cited at first use with the publication that introduced them.

\item {\bf New assets}
    \item[] Question: Are new assets introduced in the paper well documented and is the documentation provided alongside the assets?
    \item[] Answer: \answerNA{} 
    \item[] Justification: the paper does not release new assets.

\item {\bf Crowdsourcing and research with human subjects}
    \item[] Question: For crowdsourcing experiments and research with human subjects, does the paper include the full text of instructions given to participants and screenshots, if applicable, as well as details about compensation (if any)? 
    \item[] Answer: \answerNA{} 
    \item[] Justification: the paper does not involve crowdsourcing nor research with human subjects

\item {\bf Institutional review board (IRB) approvals or equivalent for research with human subjects}
    \item[] Question: Does the paper describe potential risks incurred by study participants, whether such risks were disclosed to the subjects, and whether Institutional Review Board (IRB) approvals (or an equivalent approval/review based on the requirements of your country or institution) were obtained?
    \item[] Answer: \answerNA{} 
    \item[] Justification: the paper does not involve crowdsourcing nor research with human subjects.

\item {\bf Declaration of LLM usage}
    \item[] Question: Does the paper describe the usage of LLMs if it is an important, original, or non-standard component of the core methods in this research? Note that if the LLM is used only for writing, editing, or formatting purposes and does \emph{not} impact the core methodology, scientific rigor, or originality of the research, declaration is not required.
    \item[] Answer: \answerNA{} 
    \item[] Justification: the core method development in this research does not involve LLMs as any important, original, or non-standard components.

\end{enumerate}

\end{document}